\documentclass[11pt,twoside]{article}

\usepackage{fullpage}

\usepackage{diagbox}
\usepackage{makecell}
\usepackage{booktabs}
\usepackage{tikz,tikz-3dplot}
\usepackage{amsmath}
\usepackage{amsfonts}
\usepackage{amssymb}
\usepackage{amsthm}
\usepackage{url,ifthen}
\usepackage[square,sort,comma,numbers]{natbib}
\usepackage{enumitem}
\usepackage{srcltx}
\usepackage{multirow}
\usepackage{boxedminipage}
\usepackage{nicefrac}
\usepackage{xspace}
\usepackage{graphicx}
\usepackage{color}
\usepackage{subfig}
\usepackage{colortbl}
\usepackage{setspace}
\usepackage{pgfplots}
\pgfplotsset{compat=1.12}
\usepackage{algorithm}
\usepackage[algo2e]{algorithm2e}
\usepackage[noend]{algorithmic}

\usepackage{xcolor}
\definecolor{DarkGreen}{rgb}{0.1,0.5,0.1}
\definecolor{DarkRed}{rgb}{0.5,0.1,0.1}
\definecolor{DarkBlue}{rgb}{0.1,0.1,0.5}
\definecolor{Gray}{rgb}{0.2,0.2,0.2}

\definecolor{c1}{RGB}{38, 70, 83}
\definecolor{c2}{RGB}{42, 157, 143}
\definecolor{c3}{RGB}{233, 196, 106}
\definecolor{c5}{RGB}{231, 111, 81}
\definecolor{c4}{RGB}{244, 162, 97}

\definecolor{c1}{RGB}{38, 70, 83}
\definecolor{c2}{RGB}{42, 157, 143}
\definecolor{c3}{RGB}{233, 196, 106}
\definecolor{c5}{RGB}{231, 111, 81}
\definecolor{c4}{RGB}{244, 162, 97}

\usepackage{listings}
\lstdefinestyle{mystyle}{
    commentstyle=\color{DarkBlue},
    keywordstyle=\color{DarkRed},
    numberstyle=\tiny\color{Gray},
    stringstyle=\color{DarkGreen},
    basicstyle=\footnotesize,
    breakatwhitespace=false,         
    breaklines=true,                 
    captionpos=b,                    
    keepspaces=true,                 
    numbers=left,                    
    numbersep=5pt,                  
    showspaces=false,                
    showstringspaces=false,
    showtabs=false,                  
    tabsize=2
}
\usepackage{caption}
\usepackage{hyperref}
\hypersetup{
    unicode=false,          
    pdftoolbar=true,        
    pdfmenubar=true,        
    pdffitwindow=false,      
    pdfnewwindow=true,      
    colorlinks=true,       
    linkcolor=DarkBlue,          
    citecolor=DarkGreen,        
    filecolor=DarkRed,      
    urlcolor=DarkBlue,          
    pdftitle={},
    pdfauthor={},
}

\def\draft{1}

\def\submit{0}

\ifnum\draft=1 
    
\else
    
\fi

\ifnum\submit=1
\newcommand{\forsubmit}[1]{#1}
\newcommand{\forreals}[1]{}
\else
\newcommand{\forreals}[1]{#1}
\newcommand{\forsubmit}[1]{}
\fi
\newtheorem{theorem}{Theorem}[section]

\newtheorem{lemma}[theorem]{Lemma}
\newtheorem{corollary}[theorem]{Corollary}

\theoremstyle{definition}
\newtheorem{definition}[theorem]{Definition}

\newtheorem{example}[theorem]{Example}

\newcommand{\chapterref}[1]{\hyperref[ch:#1]{Chapter~\ref{ch:#1}}}

\newcommand{\claimref}[1]{\hyperref[claim:#1]{Claim~\ref{claim:#1}}}

\newcommand{\corollaryref}[1]{\hyperref[cor:#1]{Corollary~\ref{cor:#1}}}

\newcommand{\definitionref}[1]{\hyperref[def:#1]{Definition~\ref{def:#1}}}

\newcommand{\equationref}[1]{\hyperref[eq:#1]{Equation~\ref{eq:#1}}}

\newcommand{\factref}[1]{\hyperref[fact:#1]{Fact~\ref{fact:#1}}}

\newcommand{\figureref}[1]{\hyperref[fig:#1]{Figure~\ref{fig:#1}}}

\newcommand{\tableref}[1]{\hyperref[tab:#1]{Table~\ref{tab:#1}}}

\newcommand{\itemref}[1]{\hyperref[item:#1]{Item~(\ref{item:#1})}}

\newcommand{\lemmaref}[1]{\hyperref[lem:#1]{Lemma~\ref{lem:#1}}}

\newcommand{\propref}[1]{\hyperref[prop:#1]{Proposition~\ref{prop:#1}}}

\newcommand{\propositionref}[1]{\hyperref[prop:#1]{Proposition~\ref{prop:#1}}}

\newcommand{\remarkref}[1]{\hyperref[rem:#1]{Remark~\ref{rem:#1}}}

\newcommand{\sectionref}[1]{\hyperref[sec:#1]{Section~\ref{sec:#1}}}

\newcommand{\theoremref}[1]{\hyperref[thm:#1]{Theorem~\ref{thm:#1}}}

\newcommand{\E}{\mathbb{E}}

\renewcommand{\hat}{\widehat}

\newcommand{\cE}{{\cal E}}
\newcommand{\cF}{{\cal F}}

\newcommand{\cN}{{\cal N}}

\newcommand{\cS}{{\cal S}}

\newcommand{\cX}{{\cal X}}

\newcommand{\defeq}{\stackrel{\small \mathrm{def}}{=}}
\renewcommand{\leq}{\leqslant}

\renewcommand{\geq}{\geqslant}

\newcommand{\R}{\mathbb{R}}

\usepackage{bm}

\newcommand{\Ind}{\mathbb I}

\newcommand{\ignore}[1]{}

\DeclareMathOperator*{\argmin}{arg\,min}
\DeclareMathOperator*{\argmax}{arg\,max}

\renewcommand{\epsilon}{\varepsilon}

\newcommand{\remove}[1]{}

\def\mydatewithyear{\leavevmode\hbox{\the\year/\twodigits\month/\twodigits\day}}
\def\twodigits#1{\ifnum#1<10 0\fi\the#1}

\def\twodigits#1{\ifnum#1<10 0\fi\the#1}

\makeatletter
\newtheorem*{rep@theorem}{\rep@title}
\newcommand{\newreptheorem}[2]{%
\newenvironment{rep#1}[1]{%
 \def\rep@title{#2 \ref{##1}}%
 \begin{rep@theorem}}%
 {\end{rep@theorem}}}
\makeatother

\newreptheorem{theorem}{Theorem}
\newreptheorem{corollary}{Corollary}
\newreptheorem{lemma}{Lemma}

\newcommand{\LFD}{L^{\rm FD}}
\newcommand{\LCE}{L^{\rm CE}}
\newcommand{\LPE}{L^{\rm PE}}

\newcommand{\lFD}{\ell^{\rm FD}}
\newcommand{\lCE}{\ell^{\rm CE}}
\newcommand{\lPE}{\ell^{\rm PE}}
\newcommand{\lSI}{\ell^{\rm SI}}

\newcommand{\thetahatmleit}{\widehat{\theta}_{t,i}^{\MLE}}

\newcommand{\CEset}{{\rm CE}}
\newcommand{\MLE}{{\rm MLE}}
\newcommand{\bPE}{b^{\rm PE}}
\newcommand{\RR}{\mathbb{R}}
\newcommand{\EE}{\mathbb{E}}
\newcommand{\PP}{\mathbb{P}}

\newcommand{\xit}{x_{t,i}}
\newcommand{\xti}{x_{t,i}}

\newcommand{\xt}{x_{t}}
\newcommand{\yit}{y_{t,i}}
\newcommand{\yti}{y_{t,i}}

\newcommand{\pt}{p_{t}}

\newcommand{\xonet}{x_{t,1}}

\newcommand{\xnt}{x_{t,n}}

\newcommand{\xeq}{x^\star}
\newcommand{\xeqi}{x_i^\star}
\newcommand{\xeqij}{x_{i,j}^\star}
\newcommand{\peq}{p^\star}

\newcommand{\zerov}{{\bf 0}}

\newcommand{\thetatruei}{\theta^*_{i}}

\newcommand{\vx}{x}
\newcommand{\vp}{p}
\newcommand{\vy}{y}

\newcommand{\ve}{e}

\newcommand{\vone}{\mathbf 1}
\newcommand{\LCal}{\mathcal{L}}

\newcommand{\ba}{\begin{array}}
\newcommand{\ea}{\end{array}}

\newcommand{\XCal}{\mathcal{X}}

\newcommand{\cx}{\mathcal{X}}

\newcommand{\NCal}{\mathcal{N}}

\newcommand{\util}{u}
\newcommand{\utilii}[1]{\util_{#1}}
\newcommand{\utili}{\utilii{i}}
\newcommand{\utilj}{\utilii{j}}

\newcommand{\PE}{\mathcal{PE}}
\newcommand{\thetabar}{\bar{\theta}}
\newcommand{\thetatrue}{\theta^\ast}

\renewcommand{\algorithmiccomment}[1]{\bgroup\hfill\textcolor{blue}{//~#1}\egroup}

\newcommand\numberthis{\addtocounter{equation}{1}\tag{\theequation}}
\newcommand{\Pcal}{\mathcal{P}}
\newcommand{\PCal}{\mathcal{P}}
\newcommand{\Xcal}{\mathcal{X}}

\begin{document}

\begin{center}

  {\bf{\LARGE{Learning Competitive Equilibria in Exchange Economies with Bandit Feedback}}}

\vspace*{.4in}

{\large{
\begin{tabular}{c}
Wenshuo Guo$^{\dagger}$, Kirthevasan Kandasamy$^{\dagger}$, \\Joseph E Gonzalez$^{\dagger}$, Michael I. Jordan$^{\dagger, \ddag}$, Ion Stoica$^{\dagger}$
\end{tabular}
}}
\vspace*{.3in}

\begin{tabular}{c}
$^\dagger$Department of Electrical Engineering and Computer Sciences,\\ 
$^\ddag$Department of Statistics, 
\\University of California, Berkeley\\
\end{tabular}

\vspace*{.3in}

\today

\vspace*{.2in}
\end{center}

The sharing of scarce resources among multiple rational agents is one of the classical problems in economics. In exchange economies, which are used to model such situations, agents begin with an
initial endowment of resources and exchange them in a way that is mutually beneficial until they reach a competitive equilibrium (CE). The allocations at a CE are Pareto efficient and fair. Consequently, they are used widely in designing mechanisms for fair division. However, computing CEs requires the knowledge of agent preferences which are unknown in several applications of interest. In this work, we explore a new online learning mechanism, which, on each round, allocates resources to the agents and collects stochastic feedback on their experience in using that allocation.
Its goal is to learn the agent utilities via this feedback and imitate the allocations at a CE in the long run.
We quantify CE behavior via two losses and
propose a randomized algorithm which achieves sublinear loss
under a parametric class of utilities. Empirically, we demonstrate the effectiveness of this mechanism through numerical simulations.

\section{Introduction}
\label{sec:intro}

An exchange economy (EE) is a classical micro-economic construct used to model situations where
multiple rational agents share a finite set of scarce resources.
Such scenarios arise frequently for applications in operations management, urban planning,
crowd sourcing, wireless networks, and sharing resources in
data centers~\citep{cohen1965theory, harris1982asymmetric,%
simonsen2018citizen, dissanayake2015task, georgiadis2006resource, hussain2013survey}.
In an EE, agents share a set of resources consisting of multiple resource types.
They begin with an initial endowment and then exchange these resources among themselves based
on a price system. This exchange process allows two agents to trade different resource types if they
find it mutually beneficial to do so.
Under certain conditions, continually trading in this manner results in a \emph{competitive
equilibrium} (CE), where the allocations have  desirable
Pareto-efficiency and fairness properties.
EEs have attracted much research attention,
historically since they are tractable models to study human behavior and price determination in real-world markets, and more recently for designing multi-resource fair division mechanisms~\citep{debreu1982existence, crockett2008learning,
tiwari2009competitive, budish2017course, babaioff2019fair, babaioff2021competitive}.

One of the most common use cases for fair division, which will be especially pertinent in this work,
occurs in the context of shared computational resources.
For instance, in a data center shared by an organization, we wish to allocate
resources such as CPUs, memory, and GPUs to different users who wish to share this cluster in a way that is Pareto-efficient (so that the resources are put into good use) and
fair (for long-term user satisfaction).
Here, unlike in real world economies where agents might trade with each other until they reach an
equilibrium, the equilibrium is computed using a
central mechanism (e.g. a cluster manager) based on the preferences submitted by
the agents to obtain an allocation with the above properties.
Indeed, fair division mechanisms are a staple in many popular multi-tenant cluster management frameworks used in practice, such as Mesos~\citep{hindman2011mesos}, Quincy~\citep{isard2009quincy},
Kubernetes~\citep{kubernetes}, and Yarn~\citep{yarn}.
Due to this strong practical motivation, a recent line of work has studied such fair division mechanisms for
resource sharing in a compute cluster%
~\citep{chen2018scheduling,ghodsi2013choosy,parkes2015beyond,ghodsi2011dominant},
with some of them based on exchange economies and their variants~\citep{zahedi2018amdahl,gutman2012,lai2005tycoon,varian1973equity}.

However, prior work on EEs and fair division typically assumes knowledge of the agent preferences, in the form of a
utility function which maps an allocation of the $m$ resource types to the value the agent
derives from the allocation.
For instance, in the above example, an application developer needs to quantify how well her application performs for each allocation of CPU/memory/GPU she receives.
At best, doing so requires the laborious and often erroneous task of profiling their
application~\citep{delimitrou2013paragon,misra2021rubberband},  and at worst, it can be infeasible due to practical constraints~\citep{venkataraman2016ernest,rzadca2020autopilot}.
However, having received an allocation, application developers find it easier to
report feedback about the utilities based on the performance they achieved. Moreover, in many
real-world systems, this feedback scheme can often be automated~\citep{hindman2011mesos}.

\subsection{Contributions \& summary of results}

We study a multi-round mechanism for computing CE in an exchange economy so as
to generate fair and efficient allocations when
the exact utilities are \textit{unknown} a priori. A central mechanism is used to learn the user utilities over time via feedback from the agents.
At the beginning of
each round, the mechanism generates allocations; at the end of the round,
agents report feedback on the allocation they received.
The mechanism then uses this information to better learn the preferences.
In particular, we focus on applications for fair division where a centralized mechanism can compute
an allocation of these resources on each round, say, by estimating the utilities and finding their equilibria.


In this pursuit,
we first formalize this online learning task and construct two loss functions: the first $L^{\rm
CE}$ directly builds on the definition of a CE, while the latter $L^{\rm PE}$ is motivated by the
fairness and Pareto-efficiency considerations that arise in fair division.
To make the learning problem tractable, we focus on a parametric class of utilities which include
the constant elasticity of substitution (CES) utilities which feature prominently in the
econometric literature and other application-specific utilities used in the systems literature.

We develop a randomized online mechanism which efficiently learns utilities over rounds of
allocations while simultaneously striving to achieve Pareto-efficient and fair allocations.
We show that this mechanism achieves $\widetilde{O}(\sqrt{T})$
loss for the two loss functions with both in-expectation and
high-probability upper bounds (Theorems~\ref{thm:TS-anytime} and~\ref{thm:TS-finitehorizon}), under
a general family of utility functions.
%
%
To the best of our knowledge, this is the first work that studies CE without knowledge of user
utilities; as such different analysis techniques are necessary. For instance, finding a CE is
distinctly different from a vanilla optimization task, and common strategies in bandit optimization
such as upper-confidence-bound (UCB) based algorithms
do not apply (details in~\ref{sec:alg}). Instead, our algorithm uses a
sampling procedure to balance the exploration-exploitation trade-off.
We develop new techniques both to bound the losses and to analyse the algorithm.
Finally, we corroborate these theoretical insights with empirical simulations.

\vspace{-1mm}
\subsection{Related work}
\label{sec:related_work}
\vspace{-2mm}

Our work builds on a rich line of literature at the intersection of microeconomics and machine
learning. This richness is not surprising: many real world systems are economic and multi-agent in nature, where decisions taken by or for one agent are weighed against the considerations of others,
especially when these agents have competing goals such as in resource allocation, matching
markets, and in auction-like settings.

As in this work, several works have studied online learning formulations to handle situations where the agents' preferences are not known a priori, but can be learned from repeated
interactions~\citep{dudik2017oracle,kakade2010optimal,balcan2016sample,babaioff2013multi,athey2013efficient,kandasamy2020mechanism}.
Our setting departs from these as we wish to learn agent preferences in an exchange
economy, with a focus on designing fair division mechanisms.


Since the seminal work of~\citet{varian1973equity}, fair division of multiple resource types has
received significant attention in the game theory, economics, and computer systems literature. One
of the most common perspectives on this problem is as an exchange economy (or as a Fisher market,
which is a special case of an EE). Moreover, fair allocation mechanisms have been deployed in many
practical resource allocation tasks when compute resources are shared by multiple users.
Due to space constraints, we defer a more detailed overview on this line of works in Appendix~\ref{app:EE-background}.



Notably, in all of the above cases, an important requirement for the mechanism is that agent utilities be known ahead of time. Some work has attempted to lift this limitation by making  explicit assumptions on the utility, but it is not clear that if these assumptions hold in
practice~\citep{le2020allox,zahedi2018amdahl}. Recently,~\citet{kandasamy2020online} provides a general method for learning agent utilities for fair division using feedback. However, they only study a \textit{single-resource} setting and do not do not explore multiple resource types.
Crucially, in the multi-resource setting, one agent can exchange a resource of one type for a
different type of resource from another user, so that both are better off after the exchange. 
Thus, learning in a \textit{multi-resource} setting is significantly more challenging than the single-resource case since there is no notion of exchange, and requires new analysis techniques.


\section{Background}
\label{sec:background}
\vspace{-2mm}


We first present some necessary background material on exchange economies, their competitive equilibria, and fair division mechanisms.

\subsection{Exchange economies}
\label{sec:background-ee}

In an exchange economy, we have $n$ agents and $m$ divisible resource types.
Each agent $i\in[n]$ has an endowment, $\ve_i = (\ve_{i1}, \dots, e_{im})$, where $e_{ij}$
can be viewed the amount of resource $j$ agent $i$ brings to the economy for trade.
In the shared compute cluster example, $\ve_i$ may represent agent $i$'s contribution to this
cluster. Without loss of generality we assume $\sum_{i\in[n]}\ve_{i1} = 1$ so that the space of
resources is denoted by $[0, 1]^m$.

We denote an allocation of these resources to the $n$ agents by
$x=(x_1,x_2,\dots,x_n)$, where $x_i\in[0,1]^m$ and $x_{ij}$ denote the amount of resource $j$ that
is allocated to agent $i$.
The set of all feasible
allocations is therefore $\mathcal{X}= \{x: \sum_{i=1}^m x_{ij} \leq \vone, x_{ij} \geq 0, \forall i\in[n],
j\in[m]\}$.

An agent's utility function is simply $\utili:[0, 1]^m \rightarrow [0, 1]$, where $\utili(x_i)$ represents her valuation for an allocation $x_i$ she receives. Here $\utili$ is non-decreasing, i.e., $\utili(x_i) \leq \utili(x'_i)$ for
all $x_i \leq x'_i$ element-wise (more allocations will not hurt).

In an exchange economy, agents exchange resources based on a price system. We denote a price vector by $p$, where $p\in\RR^m_+$ and $1^\top p = 1$ (the normalization accounts for the fact that only relative prices matter). Here $p_j$ denotes the price for resource $j$. Given a price vector $p$, an agent $i$ has a \emph{budget} $p^\top \ve_i$, which is the \emph{monetary} value of her endowment according to the prices in $p$. As this is an economy, a rational agent will then seek to maximize her utility under her budget:
\begin{align}
d_i(p) = \argmax_{x_i\in [0,1]^m} u_i(x_i)
\quad\text{subject to}\; p^\top x_i \leq p^\top e_i.
\label{eqn:demandset}
\end{align}
While generally, the preferred allocations $d_i(p)$ form a set, for simplicity we will assume it is a singleton
and treat $d_i$ as a function which outputs an allocation for agent $i$. This is justified under very general conditions~\citep{mas1995microeconomic,varian1992microeconomic}. We refer to $d_i(p)$ chosen in the above manner as the agent $i$'s demand for prices $p$.

\paragraph{Competitive equilibria -- definition, existence and uniqueness:}
A natural way to allocate resources to agents is to set prices $p$ for the resources,
and have the agents maximize their utility under this price system. That is, we allocate $x(p)=(x_1,\dots,x_n)$. Unfortunately, such an allocation may be infeasible, and even if it were, it may not result in an efficient allocation. However, under certain conditions, we can compute a \emph{competitive equilibrium} (CE), where the prices have both of the desired properties:

\begin{definition}[Competitive (Walrasian) Equilibrium] \label{def:ce}
A CE is a pair of allocations and prices $(\xeq, \peq)$ such that
\textbf{\emph{(i)}} the allocations are feasible
and 
\textbf{\emph{(ii)}} all agents maximize their utilities under the budget induced by prices $\peq$.
Precisely,
\begin{align*}
   \sum_{i\in[n]}\xeqij &\leq
\sum_{i\in[n]} e_{ij} = 1, \quad \forall\,j\in[m],\\
\xeqi &= d_i(\peq),
\quad \forall\,i\in[n]. 
\end{align*}
\end{definition}
Some definitions of a CE require that the first condition above being an exact equality
(e.g.,~\citep{mas1995microeconomic}). However, when the utilities are strictly increasing (which will be the case in the sequel), 
both definitions coincide~\citep{varian1992microeconomic}.

\paragraph{Utilities.}
In general, CEs do always exist but may not be unique. However, 
one important class of utilities that guarantee this condition with much attention in the fair division literature is the constant elasticity of substitution
(CES) utility.
Due to its favorable properties, CES utilities are widely-studied in many fair division works, and
most of the existing algorithms that generate fair and efficient allocations assume CES utilities or
its sub-classes~\citep{varian1992microeconomic,mas1995microeconomic}.
CES utilities are also
ubiquitous in the microeconomics literature; due to this flexibility in
interpolating between perfect substitutability and complementary, they are also able to approximate
several real-world utility functions.
Moreover, computationally, there are efficient methods for computing a
CE in the CES and related classes~\citep{zhang2011proportional,zahedi2018amdahl}.
In contrast, even when CE exist, they may be hard to find under more general classes of
utilities~\citep{varian1992microeconomic}. 


\begin{example}[CES utilities]
\label{eg:ces}
A CES utility takes the form
$\util_i(x) = \big(\sum_{j=1}^m \theta_{ij} x_i^{\rho}\big)^{1/\rho}$ where
$\rho$ is the elasticity of substitution, and $\theta_i = (\theta_{i1},\dots,\theta_{im})$
is an agent-specific parameter.
When $\rho=1$, this corresponds to linear utilities where goods are perfect substitutes.
As $\rho\rightarrow\infty$, the utilities approach perfect complements.

\end{example}


\subsection{Fair division}\label{sec:fair-division}

%

%
We describe exchange economies which are used in fair-division mechanisms. We first formally define the fair division problem.

In a standard mechanism for fair division when the utilities are inputs, each agent truthfully\footnote{%
Unlike some previous works on fair
division~\citep{parkes2015beyond,kandasamy2020online,ghodsi2011dominant},
we do not study strategic considerations, where agents may attempt to manipulate outcomes in their favor by falsely submitting their utilities.}
submits her utility $\utili$ to the mechanism. The mechanism then returns an allocation $x\in\XCal$ that are not only efficient but also fair, which satisfies the following two requirements: \emph{sharing incentive} (SI) and \emph{Pareto-efficiency} (PE). An allocation $x=(x_1,\dots,x_m)$ satisfies SI if the utility an agent receives
is at least as much as her utility when using her endowment, i.e. $\utili(x_i) \geq \utili(e_i)$.
This simply states that she is never worse off than if she had kept her endowment to herself, so she has the incentive to participate in the fair division mechanism.

A feasible allocation $x$ is said to be PE if the utility of one agent can be increased only by decreasing the utility of another. Rigorously, an allocation $x$ dominates another $x'$,
if $\utili(x_i)\geq\utilj(x'_i)$ for all $i\in [n]$ and there exists some $i\in[n]$
such that $\utili(x_i) > \utili(x'_i)$.
An allocation is Pareto-efficient if it is not dominated by any other point.
We denote the set of Pareto-efficient allocations by $\PE$.
One advantage of the PE requirement, when compared to other formalisms which 
maximize social or egalitarian welfare, is that it does not compare the utility of one agent against that of another. The utilities are useful solely to specify an agent's preferences over different allocations.
\vspace{-2mm}
\paragraph{EEs in fair division:}
The above problem description for fair division naturally renders itself to a solution based on EEs. By treating the resource allocation environment as an exchange economy, we may compute its
equilibrium to determine the allocations for each agent. Then, the SI property follows from the fact that each agent is maximizing her utility under her budget,
and an agent's endowment (trivially) falls under her budget. The PE property follows from the first theorem of welfare
economics~\citep{varian1992microeconomic,mas1995microeconomic}. Several prior works have used this connection to design fair-division mechanisms for many practical applications~\citep{varian1973equity,crockett2008learning,zahedi2018amdahl}.


\vspace{-2mm}
\paragraph{Computing a CE:}
In order to realize a CE allocation in a fair division mechanism, the mechanism needs to compute
a CE given a set of utilities.
One way to do this is via tatonnement~\citep{varian1992microeconomic}.
While there are general procedures, such as tatonnement~\citep{varian1992microeconomic},
they are not guaranteed to converge to an equilibrium even when it exists; moreover, even when they
do, the rate of convergence can be slow.
This has led to the development of efficient procedures for special classes of
functions.
One such method is proportional response
dynamics (PRD)~\citep{zhang2011proportional,zahedi2018amdahl} which converges faster
under CES utilities~\citep{zhang2011proportional} and other classes of
utilities~\citep{zahedi2018amdahl}
when
$e_i = \alpha_i \vone_m$ for all $i\in[n]$ (with $\sum_i \alpha_i = 1$). In fact, in our evaluations,  we adopt PRD for computing a CE, which is a subroutine of the learning algorithm.


We note that in the context of fair division, the CE allocations are more pertinent
than the CE prices. While the prices are
used to compute fair allocations, they are not used directly in their own right.

%

\vspace{-1mm}
\section{Online Learning Formulation}
\label{sec:setup}
\vspace{-1mm}

We formalize online learning an equilibrium in an exchange economy under bandit feedback, when the exact agent utilities are unknown a priori. We consider a multi-round setting, where in each round $t$, the mechanism selects $(\xt, \pt)$, where $\xt = (\xonet,\dots,\xnt)\in\XCal$ are the allocations for each agent for the current round, and $\pt$ are the prices for units of each resource. 

The agents, having experienced their allocation, report stochastic feedback
$\{\yti\}_{i\in[n]}$,
where $\yit$ is $\sigma$ sub-Gaussian and $\EE[\yti|\xit] = \utili(\xit)$.
The mechanism then uses this information to compute allocations for the next round.
As described in Section~\ref{sec:intro},
this set up is motivated by use cases in data center resource allocations, where jobs (agents)
cannot state their utility upfront, but can report feedback on their
performance in an automated way.



Going forward, we slightly abuse notation when referring to the allocations. When $i\in[n]$ indexes an agent, $x_i=(x_{i1}, \dots, x_{im})\in[0,1]^m$ denotes the allocation to
agent $i$.
When $t$ indexes a round, $\xt=(\xonet,\dots,\xnt) \in\Xcal$ will refer to an allocation to all
agents, where $\xti=(x_{t,i,1},\dots, x_{t,i,m})\in[0,1]^m$ denotes $i$'s allocation  in that round.
The intended meaning should be clear from context.

\vspace{-2mm}
\subsection{Losses}
\vspace{-2mm}
We study two losses for this setting. 
The first loss is based directly on the definition of an equilibrium
(Def.~\ref{def:ce}). For $a\in\RR$, denote $a^+ = \max(0, a)$. We define the CE loss $\lCE$ of an allocation--price pair $(x,p)$ 
as the sum, over all agents, of the difference between the maximum attainable utility under price
$p$ and the utility achieved by allocation $x$.
The $T$-round loss $\LCE_T$
is the sum of $\lCE(\xt,\pt)$ losses over $T$ rounds. We have:
\begin{align}
\begin{split}
\label{eqn:celoss}
\lCE (x, p) &\defeq \sum_{i=1}^n
\left(\max_{x'_i: p^\top x'_i \leq p^\top \ve_i} u_i(x'_i) - u_i(\vx)\right)^{+},\\
\LCE_T &\defeq \sum_{t=1}^T \lCE(\xt, \pt).
\end{split}
\end{align}
It is straightforward to see that for a CE pair $(\xeq,\peq)$, we have $\lCE(\xeq,\peq) = 0$.
As this loss is based directly on the definition of a CE, it captures many of the properties of a
CE.

Our second loss is motivated by the fair division use case.
Recall from Sec.~\ref{sec:fair-division} that in fair division,
 while prices are useful in computing CE allocations, they have no value in their own right.
Therefore, we will motivate our loss function based on the sharing incentive (SI) and
Pareto-efficiency (PE) desiderata for fair division. It is composed of two parts. We define the SI loss $\lSI$ for an allocation $x$ as the sum, over all agents, of how much they are worse off than their endowment utilities.
We define the PE loss $\lPE$ for an allocation $x$ as the minimum sum, over all agents,
of how much they are worse off than some Pareto-efficient utilities.
Next, we define the fair division loss $\lFD$ as the maximum of $\lSI$ and $\lPE$.
Finally, we define the $T$-round loss $\LFD_T$ for the online mechanism
as the sum of $\lFD(\xt)$ losses over $T$ rounds.
We have:
\begin{align*}
&\lSI(x) \defeq \sum_{i=1}^n \left(u_i(\ve_i) - u_i(\vx_i)\right)^{+},\\
&\lPE \defeq \inf_{x' \in \PE} \sum_{i=1}^n \left(u_i(\vx_i') - u_i(\vx_i)\right)^{+},
\\
&\lFD(x) \defeq \max\left(\lPE(x), \lSI(x)\right),\\
&\LFD_T \defeq \sum_{t=1}^T \lFD(\xt).
\label{eqn:fdloss}
\numberthis
\end{align*}
Note that individually achieving either small $\lSI$ or $\lPE$ is trivial:
if an agent's utility is strictly increasing, then by allocating all the resources to this agent
we have zero $\lPE$ as such an allocation is Pareto-efficient;
moreover, by simply allocating each agent their endowment we have zero $\lSI$.
In $\lFD$, we require both to be simultaneously small which necessitates a clever allocation
that accounts for agents' endowments and utilities.
One intuitive interpretation of the PE loss is that it can be bounded above by
the $L_1$ distance to the Pareto-front in utility space;
i.e. denoting the set of Pareto-efficient utilities by $U_{\rm PE}=\{\{\utili(x_i)\}_{i\in[n]}; x
\in\PE \}\subset \RR^n$, and letting $\util(x) = (\util_1(x_1),\dots,\util_n(x_n))\in\RR^n$,
we can write, $\lPE(x) \leq \min_{u\in U_{\rm PE}} \| u - \util(x)\|_1$.

The FD loss is more interpretable as it is stated in terms of the SI and PE requirements for fair
division. On the other hand, the CE loss is less intuitive. Moreover, in EEs, while prices help us
determine the allocations, they do not have value on their own.
Given this, the CE loss has the somewhat undesirable property that it depends on the prices $\pt$. That said, since the CE loss is based directly on the definition of a CE, it captures other properties of a CE that are not considered in $\lFD$ (see an example in Appendix~\ref{sec:apploss}). It is also worth mentioning that either loss cannot be straightforwardly bounded in terms of the
other.

Note that we have presented a basic version of the online learning
framework as it provides a simplest platform to study the learning problem of efficient and fair allocations. For instance, one could consider richer settings where the utilities might change over time with certain contextual information. While these settings are beyond the scope of this work, we believe the analysis techniques and intuitions developed here are also insightful in analysing other variant settings.

\subsection{Model and assumptions}
\label{sec:assumptions}

To make the learning problem tractable,
we make some additional assumptions on the problem.
We consider the following  parametric class of utility functions $\PCal$.

Let $\phi_j:[0,1]\rightarrow[0,1]$ be an increasing function which maps the allocation $x_{ij}$ of
resource $j$ to agent $i$ to some feature value.
For brevity, we will write 
$\phi:[0,1]^m\rightarrow [0,1]^m$, such that
$\phi(x_i) = (\phi_1(x_{i1}), \phi_2(x_{i2}), \dots, \phi_m(x_{im}))$;
Next, let $\mu:\RR_+\rightarrow[0,1]$ be an increasing function. Finally, let $\Theta\subset\RR_+^m$
be a set of positive parameters.
Then, we consider the following class of utilities $\Pcal$:
\begin{align}
\begin{split}
\PCal = \bigg\{ &\{u_i\}_{i=1}^n;\;\; 
    u_i(x_i) = \mu\left(\theta_i^\top\phi(x_i)\right)\\
    &\text{ for some }
    \theta_i\in\Theta,\;
    \forall i\in[n]
\bigg\}
\label{eqn:Pcaldefn}
\end{split}
\end{align}
An agent's utility then takes the form
$\utili(x_i) = \mu({\thetatruei}^\top\phi(x_i))$ where the featurization $\phi$ and the function $\mu$ are known, but the true parameters $\thetatruei\in\Theta$
are unknown and need to be learned by the mechanism.

We consider the above class of functions for the following reasons.
First, observe that it represents a valid class of utilities in that for all positive $\theta$,
the utilities are increasing in the allocations.
Second, a CE is guaranteed to exist uniquely in this class.
Third, from a practical point of view, it subsumes a majority of utilities studied in the
fair division literature, such as linear utilities,
the CES utilities from Example~\ref{eg:ces}~\citep{crockett2008learning,%
tiwari2009competitive,budish2017course,babaioff2019fair,babaioff2021competitive},
and other application-specific utilities~\citep{zahedi2018amdahl,venkataraman2016ernest},
Fourth, also from a practical point of view,
the CE can be efficiently computed on this class~\citep{zhang2011proportional}. 
Finally, it also allows us to leverage techniques for estimating generalized linear models
in our online learning mechanism~\citep{chen1999strong,filippi2010parametric}.

We will also assume the following regularity conditions on $\Pcal$ to avoid some degenerate cases in
our analysis.
First, $\mu$ is continuously differentiable, it is
Lipschitz-continuous with constant $L_\mu$, and
$C_\mu = \inf_{\theta \in \Theta, \vx \in \cx} \dot{\mu} \left(\theta^\top
\phi(\vx)\right) > 0$.
Second, $\Theta
\subset
[\theta_{\rm min}, \infty)^m$, where $\theta_{\rm min}>0$.
These assumptions can be relaxed (albeit with a more involved analysis),
or replaced by other equivalent
regularity conditions~\citep{chen1999strong,filippi2010parametric}, without affecting the main
analysis ideas or take-aways in this paper.
Our results also apply when $\mu$, $\phi$, and $\Theta$ can be defined separately for
each agent, but we assume they are the same to simplify the exposition.



\section{Algorithm and Theoretical Results} \label{sec:alg}


\begin{algorithm}[!ht]
\begin{algorithmic}[1]
\caption{A Randomized Alg. for Learning in EEs}
\label{Alg:TS}
\STATE \textbf{Input:} 
number of initialization sub-phases $M\geq 1$,
confidence parameters $\{\delta_{t}\}_{t\geq 1}$.
\STATE $t\leftarrow 0$
\FOR[Initialization phase]{$\ell = 1, \dots, M$}  
\label{alg:init}
    \FOR{$k=1,\dots,\max(m,n)$}
        \STATE $t\leftarrow t+1$, $x_t\leftarrow(\zerov_m,\dots,\zerov_m)$
        \label{alg:cyclestart}
        \FOR{$h=1,\dots,\min(m,n)$}
            \IF{$m<n$}
            \STATE $x_{t, h+k-1,j} \leftarrow 1$ for all $j\in[m]$.
            \ELSE
            \STATE $x_{t, i, h+k-1} \leftarrow 1$ for all $i\in[n]$.
            \ENDIF
        \ENDFOR
        \STATE Allocate $x_t$ and observe rewards $\{\yti\}_{i\in[n]}$.
        \label{alg:cycleend}
    \ENDFOR
\ENDFOR
\WHILE[Round for learning phase]{\textbf{True}}
\STATE $t\leftarrow t+1$

\FOR{$i=1,2,\ldots, n$}
\STATE Compute $Q_{t,i} \defeq \sum_{s = 1}^{t-1} \phi(\vx_{s,i}) \phi(\vx_{s,i})^\top$
\label{alg:Qit}
\STATE Compute\\ \scalebox{0.82}{$\bar{\theta}_{t,i} = \argmin_{\theta \in \Theta} \bigg\|\sum_{s=1}^{t-1} \phi(\vx_{s,i})
\Big(\mu (\theta^\top \phi(\vx_{s,i})) - y_{s,i}\Big)\bigg\|_{Q_{t,i}^{-1}}$}
\label{alg:thetabarit}
\STATE Sample $\theta'_{t,i} \sim \NCal(\thetabar_{t,i}, \alpha_t^2 Q_{t,i}^{-1})$.
\label{alg:step-sampling} \COMMENT{See~\eqref{eq:alpha} for $\alpha_t$.}
\STATE $\theta_{t,i} \leftarrow \argmin_{\theta'\in\Theta} \|\theta'_{t,i} - \theta'\|$.
\label{alg:proj-sample}
\COMMENT{Projection} 
\ENDFOR
\STATE Choose allocations and prices $\xt,\pt$ =
    $\CEset(\{u_{t,i}\}_{i=1}^n)$, where $u_{t,i}(\cdot) =
\mu(\theta_{t,i}^\top \phi(\cdot))$
\label{alg:ce}
\STATE Observe rewards $\{\yti\}_{i\in[n]}$.
\ENDWHILE
\end{algorithmic}
\end{algorithm}

We present a randomized online learning algorithm for learning the agents' utilities and generating
fair and efficient allocations.
Note that this algorithm not only needs to learn the unknown
utilities quickly, but should also simultaneoulsy find the CE allocation.
This latter aspect introduces new challenges in our setting.
For instance, the most popular approach for stochastic optimization under bandit feedback 
are based on upper-confidence-bounds (UCB).
However, finding a CE cannot be straightforwardly framed as a vanilla optimization procedure and
hence UCB procedures do not apply.
Instead, our proposed algorithm uses a key randomized sampling step, which tradeoffs between
exploration and exploration while maintaining the utilities' shape constraints in every round for
computing the CE (details in proof sketch).

The algorithm, outlined in Algorithm~\ref{Alg:TS}, takes input parameters $M$
and $\{\delta_{t}\}_{t\geq 1}$
whose values we will specify shortly. It begins with an initialization phase for $M$ sub-phases
(line~\ref{alg:init}), each of length $\min(n,m)$.
During each sub-phase, we allocate each resource entirely to each user for at least one round.
This initialization phase ensures that some matrices we define
subsequently are well conditioned.

After the initialization phase, the algorithm operates on each of the remaining rounds as follows.
For each user, it first computes quantities $Q_{t,i}\in\RR^{m\times m}$ and
$\thetabar_{t,i}\in\RR^m$ as defined in lines~\ref{alg:Qit}, and~\ref{alg:thetabarit}.
As we explain shortly, $\thetabar_{t,i}$ can be viewed as an estimate of
$\thetatrue_i$ based on the data from the first $t-1$ rounds.
The algorithm then samples
$\theta'_{t,i}\in\RR^m$ from a normal distribution with mean $\thetabar_{t,i}$ and co-variance
$\alpha^2_t Q_{t,i}$, where, $\alpha_t$ is defined as:
\begin{align}
\label{eq:alpha}
\begin{split}
    \alpha_t^2 &= 4 \frac{\kappa^2 \sigma^2}{C_\mu^2} m\log (t) \log\left(\frac{m}{\delta_{t}}\right),\\
\kappa &= 3 + 2\log\left(1 + 2\|\phi(\vone)\|_2^2\right).
\end{split}
\end{align}
The sampling distribution, which is centered at our estimate $\thetabar_{t,i}$,
 is designed to balance the exploration-exploitation trade-off on this problem.
Next, it projects the sampled $\theta'_{i,t}$ onto $\Theta$ to obtain
$\theta_{t,i}$.

In line~\eqref{alg:ce}, the algorithm obtains an allocation and price pair
$\xt,\pt$ by computing the CE on the $\theta_{t,i}$ values obtained above,
i.e. by pretending that $u_{t,i}(\cdot) = \mu(\theta_{t,i}^\top \phi(\cdot))$ is the utility
for user $i$. 

It is important to note that the computation of the CE happens as a \textit{subroutine} of the
mechanism, and users will simply receive the allocations $\xt$. The mechanism collects the rewards
$\{y_{t,i}\}_{i\in[n]}$ from each user and then repeats the same for the remaining rounds. As we discussed in Sec.~\ref{sec:fair-division}, there are different ways to compute a CE efficiently in our setting, including tatonnement or the proportional response dynamics (PRD) algorithm~\citep{zhang2011proportional} which we implemented. Given that our algorithm focus on learning the efficient and fair allocations, we do not focus on the computation complexity of CE in this work. Empirically, we find PRD converges quickly in the simulations.




\paragraph{Computation of $\thetabar_{it}$:}
It is worth explaining steps~\ref{alg:Qit}--\ref{alg:thetabarit} used to obtain the
estimate $\thetabar_{it}$ for user $i$'s parameter $\thetatrue_i$.
Recall that for each agent $i$, the
mechanism receives stochastic rewards $y_{t,i}$ where $y_{t,i}$ is a $\sigma$ sub-Gaussian random
variable with $\E[y_{t,i}] = u_i(\vx_{t,i})$ in round $t$.
Therefore, given the allocation-reward pairs $\{(\vx_{s,i}, y_{s,i})\}_{s=1}^{t-1}$,
the maximum quasi-likelihood estimator $\thetahatmleit$ for
$\theta_i$ is defined as the maximizer of the quasi-likelihood
$\LCal(\theta) = \sum_{s=1}^{t-1} \log p_{\theta}(y_{s,i}|\vx_{s,i})$, where
$p_{\theta}(y_{i}|\vx_{i})$ is as defined below.
Here, $\mu(\nu) = \frac{\partial b(\nu)}{\partial \nu}$ and
$c(\cdot)$ is a normalising term.
We have:
\begin{align}\scalebox{0.95}{$
p_{\theta}(y_{i}|\vx_{i}) =
\exp\left(y_i \theta^\top \phi(\vx_{si}) - b(\theta^\top \phi(\vx_{si})) + c(y_i)\right).$}
\label{eqn:expfamily}
\end{align}
Upon differentiating, we have that $\thetahatmleit$  is the
unique solution of the estimating equation:
\begin{align*}
\sum_{s=1}^{t-1} \phi(\vx_{si})\left(\mu \left(
\big(\thetahatmleit\big)^\top \phi(\vx_{s,i})\right) - y_{si}\right) = 0.
\end{align*}
In other words, $\theta_{t,i}^{\MLE}$ would be the maximum likelihood estimate for $\thetatruei$
if the rewards $y_{t,i}$ followed an exponential family likelihood as shown in~\eqref{eqn:expfamily}.
Our assumptions are more general; we only assume the rewards are sub-Gaussian centred at
$\mu({\thetatruei}^\top\phi(\xit))$.
However, this estimate is known to be consistent under very general conditions, including when the
rewards are sub-Gaussian~\citep{chen1999strong,filippi2010parametric}.
Since $\hat \theta_{t,i}^{\MLE}$ might be outside of the set of feasible parameters
$\Theta$, this motivates us to perform the projection in the $Q^{-1}_{t,i}$ norm
to obtain $\thetabar_{t,i}$ as defined in line~\ref{alg:thetabarit}.
Here,  $Q_{t,i}$, defined in line~\eqref{alg:Qit},
is the design matrix obtained from the data in the first $t-1$ steps.




\paragraph{On the algorithm design:}

It is worth comparing the design of our algorithm against prior work in the bandit
literature under similar parametric
assumptions~\citep{dani2008stochastic,filippi2010parametric,li2017provably,rusmev2010linearly}.
For instance, in a CE, each agent is maximizing their utility under a budget constraint.
Therefore, a seemingly natural idea is to adopt a UCB based procedure,
which is the most common approach for stochastic
optimization under bandit feedback~\citep{auer2002using}.
However, adopting a UCB-style method for our problem proved to be unfruitful.
Consider using a UCB of the form
$\mu(\thetabar_{it}^\top \phi(\cdot)) + U_{it}(\cdot)$, where $U_{it}$
quantifies the uncertainty in the current estimate.
Unfortunately, a CE is not guaranteed to exist for utilities of the above form, which
means that finding a suitable allocation can be difficult.
An alternative idea is to consider UCBs of the form
$\mu(\widehat{\theta}_{t,i}^\top \phi(\cdot))$
where $\widehat{\theta}_{t,i}$ is an upper confidence bound on $\thetatruei$
(recall that both $\thetatruei$ and $\phi$ are non-negative).
While CEs are guaranteed to exist for such UCBs,
$\widehat{\theta}_{t,i}$ is not guaranteed to uniformly converge to  $\thetatruei$, resulting
in linear loss.

Instead, our algorithm takes inspiration from classical Thompson sampling (TS) procedure
for multi-armed bandits in the Bayesian paradigm~\citep{thompson1933likelihood}.
The sampling step in line~\ref{alg:step-sampling} is akin to sampling from the posterior
beliefs in TS.
It should be emphasized that the sampling distributions on each round cannot be interpreted
as the posterior of some prior belief on $\thetatruei$.
In fact, they  were designed so as to put most of their mass inside
a frequentist confidence set for $\thetatruei$.

\subsection{Upper bounds on the loss}


The following two theorems are the main results bounding the loss terms
$\LFD, \LCE$ for Algorithm~\ref{Alg:TS}.
In the first theorem, we are given a target failure probability of at most $\delta$.
By choosing $\delta_{t}$ appropriately, we obtain an infinite horizon algorithm for which
both loss terms are $\widetilde{O}(\sqrt{T})$ with probability at least $1-\delta$.
In the second theorem, with a given time horizon $T$, we obtain an algorithm whose expected losses are $\widetilde{O}(\sqrt{T})$.


\vspace{0.05in}
\begin{theorem}\label{thm:TS-anytime} Assume the conditions in Section~\ref{sec:assumptions}.
Let $\delta > 0$ be given. Choose $\delta_{t} = \frac{2\delta}{n\pi^2 t^2}$.
Then, 
the following upper bounds on $\LFD, \LCE$  hold for 
Algorithm~\ref{Alg:TS} 
with probability at least $1-\delta$.
\begin{align*}\scalebox{0.88}{$
\LFD(T),\, \LCE(T) \;\in\;
    O\left(n\left(m + \frac{m^2}{\sqrt{M}}\right)\sqrt{T}\big(\log(nT/\delta) + \log(T)\big)\right).$}
\end{align*}
\end{theorem}

\vspace{0.05in}
\begin{theorem}\label{thm:TS-finitehorizon}
Assume the conditions in Section~\ref{sec:assumptions}.
Let $T > M\max(m^2, n)$ be given. Choose $\delta_{t} = \frac{1}{T}$.
Then, 
the follow upper bounds on $\LFD, \LCE$  hold for 
Algorithm~\ref{Alg:TS}.
\begin{align*}\scalebox{0.9}{
 $\EE[\LFD(T)], \,\EE[\LCE(T)] \,\in\;
    O\left(n\left(m + \frac{m^2}{\sqrt{M}}\right)\sqrt{T} \left(\log(T)\right) \right),$}
\end{align*}
\end{theorem}


Above, probabilities and expectations are with respect to both
the randomness in the observations and the sampling procedure.
Both theorems show that we can learn with respect to both losses at
$\sqrt{T}$ rate.
Note that the rates depend on the number of initialization sub-phases $M$.
By choosing $M=m^2$, we get a $\widetilde{O}(nm\sqrt{T})$ bound.
However, this also requires a large initialization phase, which may not be feasible in practice.
We can instead choose $M$ to be small, but this leads to correspondingly worse asymptotic bounds.


\begin{figure*}[!ht]
\centering
\begin{tabular}{ccc} 
\includegraphics[width=0.27\textwidth]{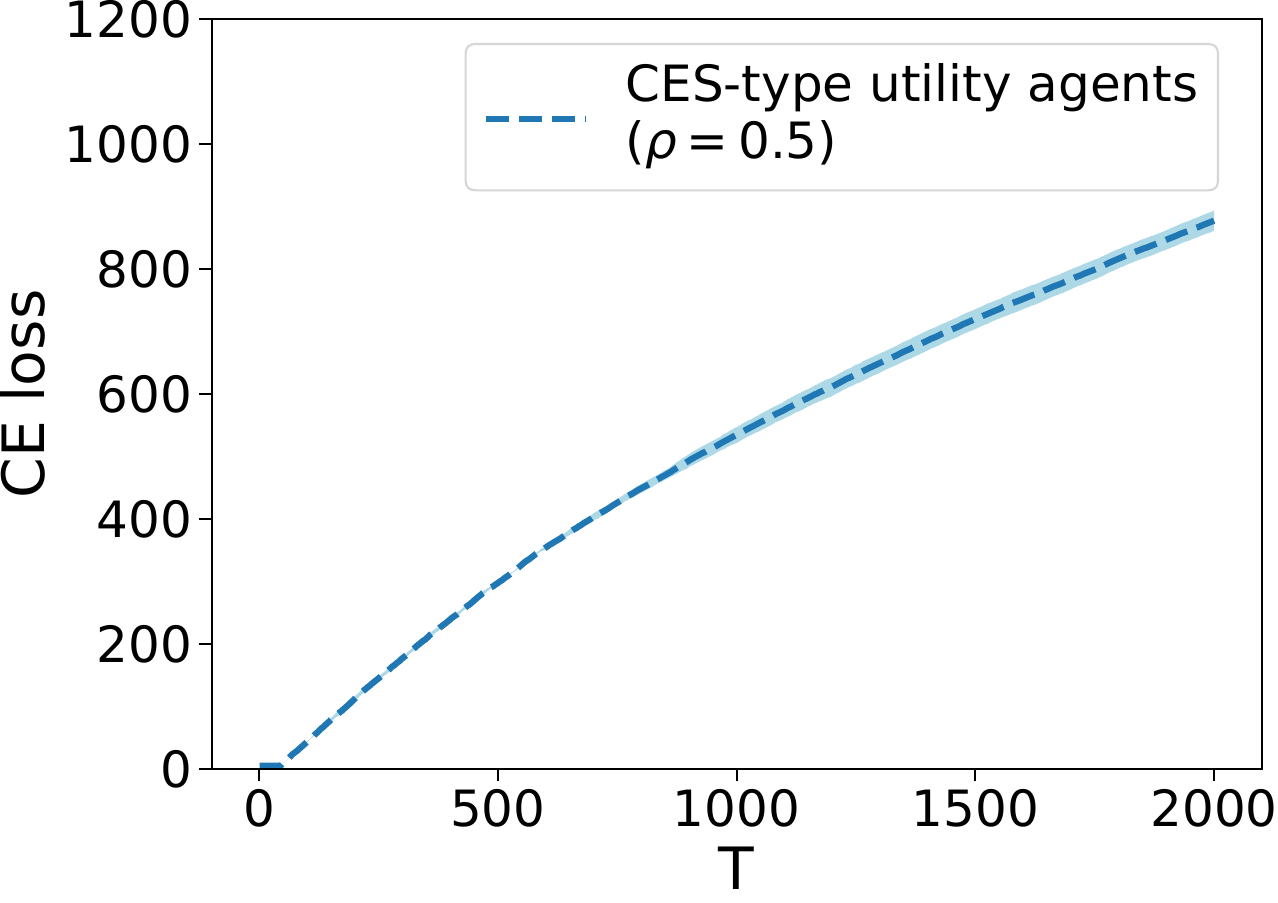} &
\includegraphics[width=0.27\textwidth]{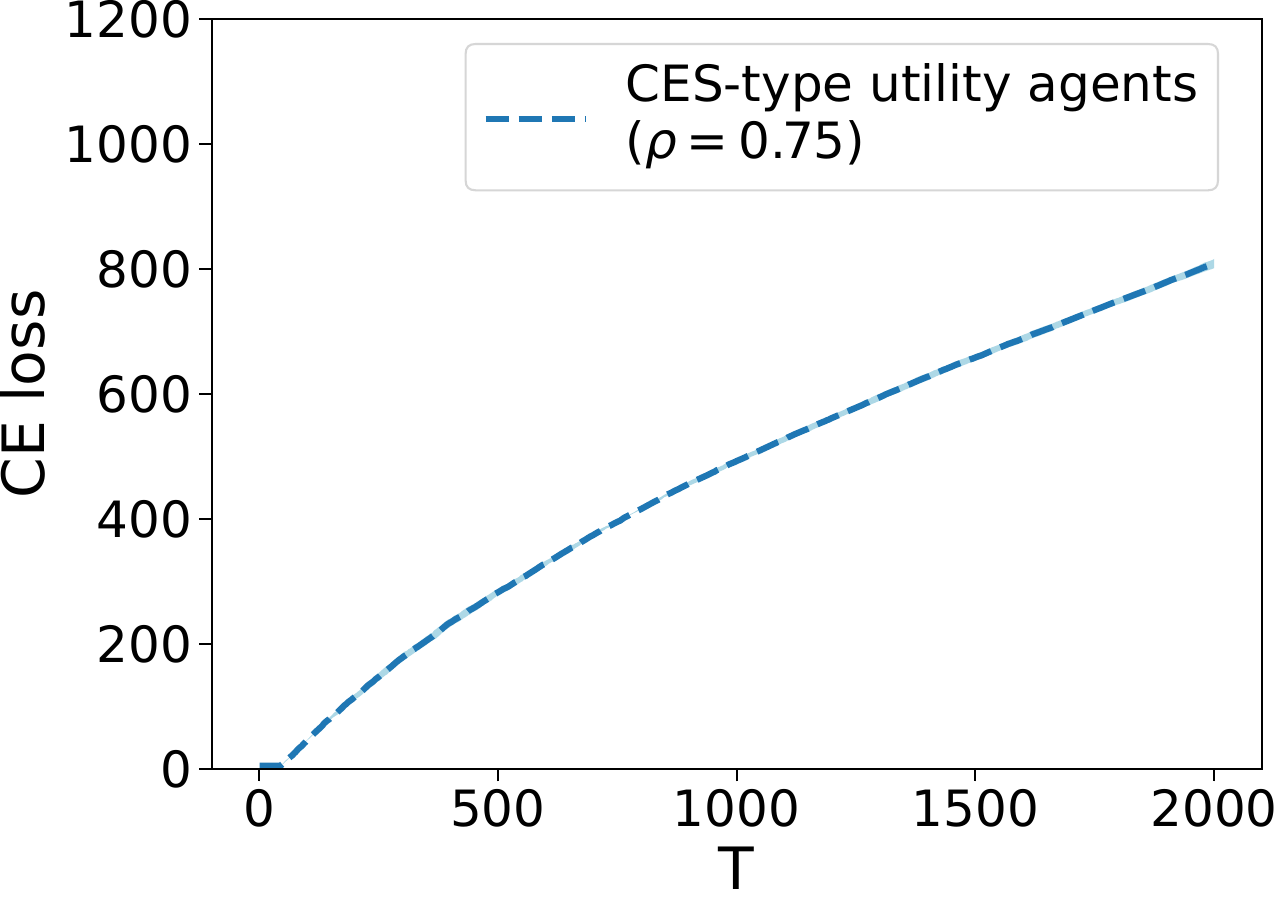} &
\includegraphics[width=0.27\textwidth]{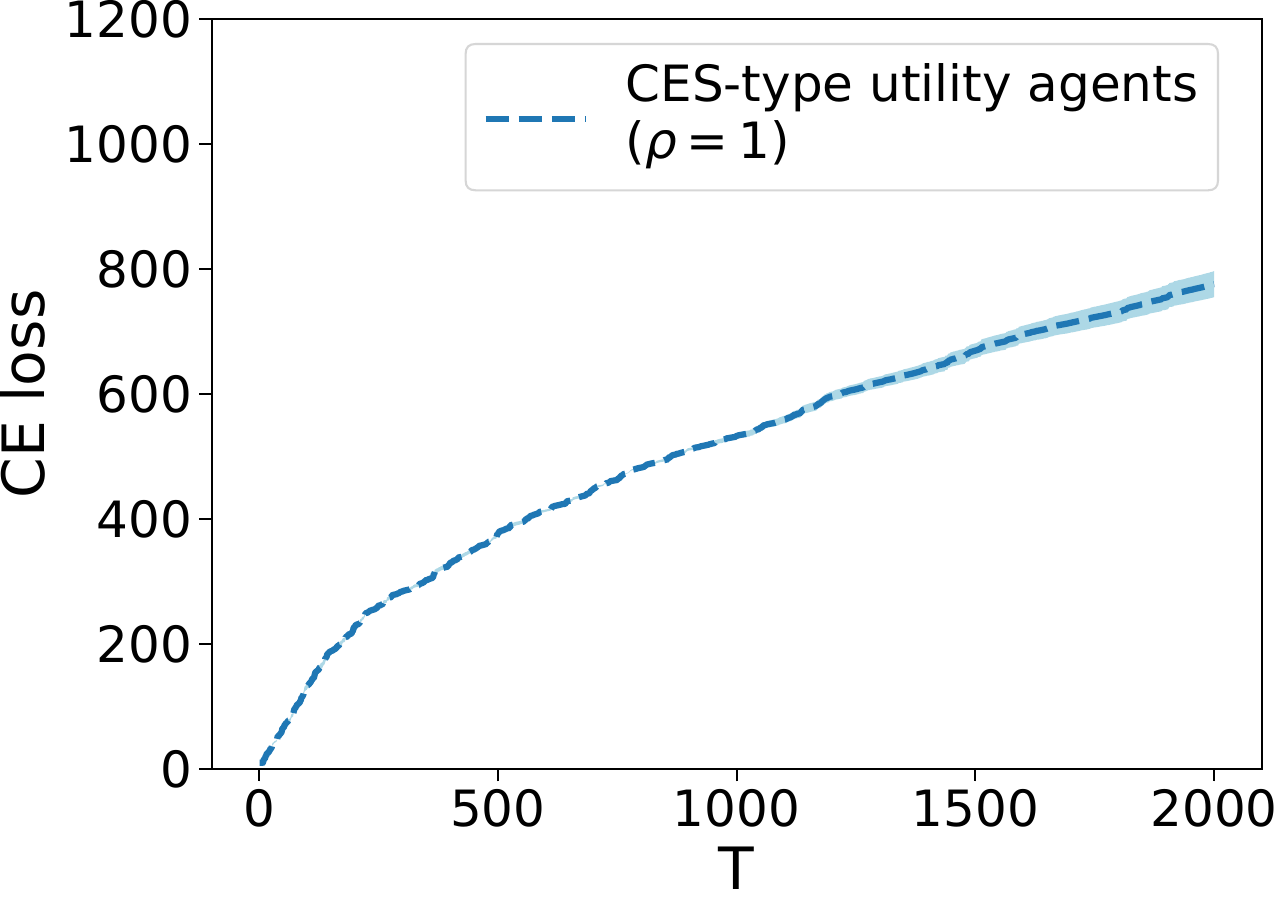} 
\end{tabular}
\caption{
The CE loss $\LCE_T$ vs the number of rounds $T$, evaluated with $m=3$ resource types and 
$n=5$ agents with CES utilities. We present results for $\rho=0.5$, $\rho=0.75$, and
$\rho=1$ respectively (see Example~\ref{eg:ces}).
All figures show results which are averaged over 10 runs, and the shaded region shows the standard
error at each time $T$.
\label{fig:results}
} \vspace{-1em}
\end{figure*}

\paragraph{Proof sketch.}
Our proof uses some prior martingale concentration results from the bandit
literature~\citep{rusmev2010linearly,filippi2010parametric}, and additionally, we use
 some high level intuitions from prior frequentist analyses of Thompson
sampling~\citep{kaufmann2012thompson,agrawal2012analysis,mutny2019efficient}.
At the same time, we also require novel techniques, both to bound the loss terms, and analyse the
algorithm.
Our proof for bounding $\LCE_T$ first defines high probability events $A_{t,i}, B_{t,i}$ for each agent
$i$ and round $t$.
$A_{t,i}$ captures the event that the estimated $\thetabar_{t,i}$ is close to $\thetatruei$
in $Q_{t,i}$ norm.
We upper bound $\PP(A_{t,i}^c)$ using the properties of the maximum quasi-likelihood estimator on
GLMs~\citep{chen1999strong,filippi2010parametric} and a martingale argument.
$B_{t,i}$ captures the event that the sampled $\theta_{t,i}$ is close to $\thetabar_{t,i}$ in
$Q_{t,i}$ norm.
Given these events, we then bound the instantaneous losses $\lCE(\xt, \pt)$ by a
super martingale with bounded differences.
The final bound is obtained by an application of the Azuma inequality.

Another key ingredient in this proof is to show that the sampling step also explores
sufficiently--the $B_{it}$ event only captures exploitation; since the sampling distribution is
a multi-variate Gaussian, this can be conveniently argued using an upper bound on the
standard normal tail probability.
While bounding $\LFD$ uses several results and techniques as above,
it cannot be directly related to $\LCE$, and requires a separate analysis.

\vspace{-2mm}
\section{Experiments}\label{sec:experiments}
\vspace{-1mm}

We evaluated Algorithm~\ref{Alg:TS} with simulations. To the best of our knowledge, this is the first online algorithm studying fair and efficient allocations with unknown utilities with multiple heterogeneous resource types, and there are no existing natural baselines. There is also no straightforward adaptation of the method described in~\citet{kandasamy2020online} for single resource types since they do not consider the exchange of resources.
We evaluated based on two types of utilities.

\paragraph{1. CES utilities:} Described in Example~\ref{eg:ces}.
\vspace{-3mm}
\paragraph{2. Amdahl's utilities:} The Amdahl's utility function,
described in~\citet{zahedi2018amdahl}, is used to model the performance of jobs distributed
across heterogeneous machines in a data center. 
This utility is motivated by Amdahl’s Law~\citep{amdahl2013computer}, which models
a job's speed up in terms of the fraction of work that can be parallelized.
Let $0<f_{ij}<1$ denote the parallel fraction of
user $i$'s job on machine type $j$.
Then, an agent's Amdahl utility is: $u_i(x) =\sum_{j=1}^m \theta_{ij}\phi_{ij}(x_{ij})$, where $\phi_{ij}(x_{ij}) = \nicefrac{x_{ij}}{f_i + (1-f_i) x_{ij}}$. $\phi_{ij}(x_{ij})$ is the relative speedup produced by allocation $x_{ij}$.
Both CES and Amdahl utilities belong to our class $\Pcal$ given in~\eqref{eqn:Pcaldefn}.

We focus our evaluation on the CE loss; computing the FD loss is computationally expensive as it requires taking an infimum over the Pareto-front (more details in Appendix~\ref{app:sec:additional-discussion}). Our first set of experiments consider an environment with $m=3$ resource types and $n=5$ agents,
all of whom have CES utilities.
We conduct three experiments with different values for the elasticity of substitution
$\rho$.
Our second set of experiments  consider an environment with $m=2$ resource types and $n=8$ agents,
all of whom have Amdahl's utilities, where the results are similar and thus included in Appendix~\ref{app:exp-additional}.
We conduct three experiments with different values for the parallel fraction $f_{ij}$.
All experiments are run for $T=2000$ rounds, where we 
set $\delta=\frac{1}{T}$. The results are given in Figure~\ref{fig:results}.
They show that the CE loss grows sublinearly with $T$ which indicates that the algorithm is able to
learn utilities and compute a CE.




To compute the CE at line~\ref{alg:ce} of Algorithm~\ref{Alg:TS}, we use the proportional
response dynamics procedure from~\cite{zhang2011proportional} with 20 iterations.
To compute $\LCE$, we need to maximize each agent's utility subject to a budget. Full experimental details and additional results are included in Appendix~\ref{app:exp-additional}.



\section{Conclusion} \label{sec:conclu}
\vspace{-1mm}

We introduced and studied the problem of online learning a competitive equilibrium in an exchange
economy, without a priori knowledge of agents' utilities.
We quantify the learning performance via two losses, the first motivated
from the definition of an equilibrium, and the second
by fairness and Pareto-efficiency considerations
in fair division. We develop a randomized algorithm which achieves $\Tilde{O}(nm\sqrt{T})$ loss
after $T$ rounds under both losses, and corroborate these theoretical results
with simulations.
While our work takes the first step towards sequentially learning a market equilibrium in
exchange economies, an interesting avenue for future work would be to study learning approaches in broader classes of agent utilities and market dynamics. 

\subsubsection*{Acknowledgments}

The authors would like to thank Mihaela Curmei and Serena Wang for the extensive discussions and helpful suggestions on an early version of the work that significantly improved the paper. The authors also thank Zhuoran Yang for their careful reading of a draft of this paper, and for making several important suggestions. This work was supported in part by the Mathematical Data Science program of the Office of Naval Research under grant number N00014-18-1-2764. WG acknowledges support from a Google PhD Fellowship.

\bibliography{ref}
\bibliographystyle{plainnat}

\clearpage
\appendix

\section{Technical Lemmas}


We first provide some useful technical lemmas.

\begin{lemma}\label{lem:TL1}
Suppose that $Z$ is a $\chi^2_m$ random variable, i.e. $Z = \sum_{k=1}^m Z_k^2$, where for all $k$, $Z_k$ are i.i.d. random variables drawn from $\cN(0,1)$. Then, 
\begin{equation*}
P(Z > m+\alpha) = \begin{cases}
e^{-\frac{\alpha}{8}} &\text{$\alpha > m$}\\
e^{-\frac{\alpha^2}{8m}} &\text{$\alpha \leq m$}.
\end{cases}
\end{equation*}
\end{lemma}
\begin{proof}
Suppose that $X$ is sub-exponential random variable with parameters $(\nu, b)$ and expectation
$\mu$.
Applying well known tail bounds for sub-exponential random variables
(e.g. \citep{wainwright2019high}) yields: 
\begin{equation*}
P(X > \mu+\alpha) = \begin{cases}
e^{-\frac{\alpha^2}{2\nu^2}} &\text{$0 \leq \alpha \leq \frac{\nu^2}{b}$, and}\\
e^{-\frac{\alpha}{2b}} &\text{$\alpha > \frac{\nu^2}{b}$}.
\end{cases}
\end{equation*}
The lemma follows from the fact that a $\chi^2_m$ random variable is sub-exponential with
parameters $(\nu, b) = (2,4)$. 
\end{proof}

\begin{lemma}(Lower bound for normal distributions)\label{lem:TL2}
	Let $Z$ be a random variable $Z\sim \cN(0,1)$, then $P(Z > t) \geq \frac{1}{t + \sqrt{t^2+4}}\sqrt{\frac{2}{\pi}} e^{-\frac{t^2}{2}}$.
\end{lemma}

\begin{proof}
First, from~\citet{abramowitz1988handbook} (7.1.13) we have, 
\[
e^{x^2}\int_x^{\infty} e^{-t^2} dt \geq \frac{1}{x+\sqrt{x^2+2}}.
\]
Set $t = \sqrt{2}x$, then the above equation yields:
\[
P(Z > t) = \frac{1}{\sqrt{2\pi}} \int_t^\infty e^{-\frac{x^2}{2}} dx \geq \frac{1}{t + \sqrt{t^2+4}}\sqrt{\frac{2}{\pi}} e^{-\frac{t^2}{2}},
\]
which completes the proof.
\end{proof}

\begin{lemma}(Azuma-Hoeffding inequality\citep{wainwright2019high})\label{lem:TL3} Let $(Z_s)_{s\geq 0}$ be a super martingale w.r.t. a filtration $(\cF_t)_{t\geq 0}$. Let  $(B_t)_{t\geq 0}$ be predictable processes w.r.t. $(\cF_t)_{t\geq 0}$, such that $|Z_s - Z_{s-1}| \leq B_s$ for all $s \geq 1$ almost surely. Then for any $\delta > 0$, 
\[
	P\left(Z_T-Z_0 \leq \sqrt{2\log\left(\frac{1}{\delta}\right)\sum_{t=1}^T B_t^2}\right) \geq 1-\delta.
\]
\end{lemma}
\begin{lemma}\label{lem:TL4}
	$\forall x \in [0,c], c>0$, we have $ x \leq \frac{c}{\log(1+c)}\log(1+x)$.
\end{lemma}
\begin{proof}
The result follows immediately from the fact that the function $f(x) \defeq \frac{x}{\log(1+x)}$ is non-decreasing on $(0, \infty)$.	
\end{proof}

\begin{lemma}(Lemma 1, \citet{filippi2010parametric})\label{lem:TL5}
	Let $(\cF_k, k\geq 0)$ be a filtration, $(m_k; k\geq 0)$ be an $\R^d$-valued stochastic process adapted to $(\cF_k)$. Assume that $\eta_k$ is conditionally sub-Gaussian in the sense that there exists some $R>0$ such that for any $\gamma \geq 0$, $k \geq 1$, $\E[\exp(\gamma \eta_k)|\cF_{k-1}] \leq \exp \left(\frac{\gamma^2 R^2}{2}\right)$ almost surely. Then, consider the martingale $\xi_t = \sum_{k=1}^t m_{k-1} \eta_k$ and the process $M_t = \sum_{k=1}^t m_{k-1}m_{k-1}^\top$. Assume that with probability one, the smallest eigenvalue of $M_d$ is lower bounded by some positive constant $\lambda_0$, and that $\|m_k\|_2 \leq c_m$ almost surely for any $k \geq 0$. Then, the following holds true:
	for any $0 < \delta < \min(1, d/e)$ and $t > \max(d, 2)$, with probability at least $1 - \delta$, 
	\[
	\|\xi_t\|_{M_t^{-1}} \leq \kappa R\sqrt{2d\log(t) \log(d/\delta)},
	\]
	where $\kappa = \sqrt{3+2\log(1 + 2\frac{c_m^2}{\lambda_0})}$.
\end{lemma}

\section{Bounding $\LCE$}\label{sec:proof-TS-anytime-LCE}

First, consider any round $t$. 
We will let $\cF_t \defeq \sigma\left(\{(\vx_{is}, y_{is})_{i=1, s=1}^{n, t-1}\}\}\right)$
denote the $\sigma$-algebra generated by the observations in the first $t-1$ rounds.
Clearly, $\{\cF_t\}_{t\geq 0}$ is a filtration.
We will denote
$\mathbb{E}_t[\cdot|\cF_t]=\mathbb{E}_t[\cdot]$ to be the expectation when conditioning on the past
observations up to round $t-1$. Similarly, define $P_t(\cdot) \defeq P(\cdot|\cF_t) =
\E[\Ind(\cdot)|\cF_t]$.

Recall that $\{\delta_t\}_{t\geq 0}$ are inputs to the algorithm.
Similarly, let $\{\delta_{2t}\}_{t\geq 0}$ be a sequence.
We will specify values for both sequences later in this proof.
Given these, further define the following quantities on round $t$:
\begin{align*}
    \beta_{1t} & \defeq \frac{2}{C_\mu} \kappa \sigma \sqrt{2m\log(t)} \sqrt{\log\left(\frac{m}{\delta_{t}}\right)} \\
    \beta_{2t} & \defeq \sqrt{\alpha_t (m  +\gamma_{2t})}
    \hspace{0.4in}\text{where, } \hspace{0.2in}
    \gamma_{2t} \defeq \max\left(8 \log\left( \frac{1}{\delta_{2t}} \right), \sqrt{8m\log\left(
\frac{1}{\delta_{2t}} \right)}\right)
        \\
    \beta_{3t} & \defeq L_{\mu} (\beta_{1t} + \beta_{2t}).
\end{align*}
Here, recall that $L_\mu$ is the Lipschitz constant of
$\mu(\cdot)$, $C_\mu$ is such that $C_\mu \defeq \inf_{\theta \in \Theta, \vx \in \cX} \dot{\mu}
\left(\theta^\top \phi(\vx)\right)$,
and $\alpha_t$ is a sequence that is defined and used
in Algorithm~\ref{Alg:TS}.

Next, we consider the following two events:
\begin{align*}
    A_{it} &\defeq \{\|\theta_{i}^\ast - \bar \theta_{it}\|_{Q_{it}} \leq \beta_{1t}\}, \\
    B_{it} &\defeq \{\|\bar \theta_{it} - \theta_{it}\|_{Q_{it}} \leq \beta_{2t}\}.
\end{align*}
where $Q_{it} \defeq \sum_{s = 1}^{t-1} \phi(\vx_{is}) \phi(\vx_{is})^\top$ is a design matrix that corresponding to the first $t-1$ steps.

Lastly, define 
\[
\rho_{it}(\vx) \defeq \|\phi(\vx)\|_{Q_{it}^{-1}}  = \sqrt{\phi^\top(\vx) Q_{it}^{-1} \phi(\vx)},
\]
and 
\[
S_{it} \defeq \left\{\vx \in \cX: u_i( \vx_{it}^\ast) - u_i(\vx) \geq \beta_{3t} \rho_{it} (\vx) \right\},\]
where $\vx_{it}^\ast = \argmax_{\vy \in \cX, \vp_t^\top \vy \leq \vp_t^\top \ve_i} u_i(\vy)$. Here, we used $\cX$ to denote the set of feasible allocations for one agent: $\{\vx \in \R^m: 0 \leq \vx \leq \vone\}$.

Intuitively, $\tilde \vx_{it}$ is the best true optimal affordable allocation for agent $i$ in round $t$ under the price function $\vp_t$. Since the set $\{\vy \in \cX, \vp_t^\top \vy \leq \vp_t^\top \ve_i\}$ is a compact set, the maximum is well defined.

Now we begin our analysis with the following lemmas.

\begin{lemma}\label{lem:lemma1}
    For any round $t > t_0$, $P_{t} (A_{it}) \geq 1 - \delta_{1t}$. 
\end{lemma}
\begin{proof}  
Define function $g_{it}(\theta) = \sum_{s=1}^{t-1} \mu\left(\theta^\top \phi(\vx_{is})\right)\phi(\vx_{is})$. Then by the fundamental theorem of calculus, we have
\[
g_{it}(\theta_{it}) - g_{it}(\theta_{it}) = G_{it} (\theta_{i}^\ast - \bar \theta_{it}),
\]
where $G_{it} = \int_{0}^1 \nabla g_{it} \left( s \theta_{i}^\ast + (1-s)\bar \theta_{it} \right) ds$, and
\[
\nabla g_{it}(\theta) = \sum_{s=1}^{t-1} \phi(\vx_{is})\phi(\vx_{is})^\top \mu'(\theta^\top \phi(\vx_{is})).
\]
By the definition of $C_\mu$ and $Q_{it}$, we have that $G_{it} \succeq C_\mu Q_{it} \succeq M \cdot I$,
where
the last inequality follows due to the initialisation scheme.
Therefore, $G_{it}$ is invertible and moreover, 
\begin{equation}\label{eq:lem-1-eq-1}
    G_{it}^{-1} \preceq \frac{1}{C_\mu} Q_{it}^{-1}.
\end{equation}
We can write,
\begin{equation}\label{eq:lem-1-eq-2}
    \theta_{i}^\ast - \bar \theta_{it} = G_{it}^{-1} \left( g_{it}(\theta_{it}) - g_{it}(\bar \theta_{it})\right).
\end{equation}
Therefore, we have,
\begin{align*}
    &\left(\theta_i^\ast - \bar \theta_{it}\right)^\top Q_{it}   \left(\theta_i^\ast - \bar \theta_{it}\right)\\
    &=  \left( g_{it}(\theta_{it}) - g_{it}(\bar \theta_{it})\right)^\top G_{it}^{-1} Q_{it} G_{it}^{-1}  \left( g_{it}(\theta_{it}) - g_{it}(\bar \theta_{it})\right)\\
    &\leq \frac{1}{C_\mu^2}  \left( g_{it}(\theta_{it}) - g_{it}(\bar \theta_{it})\right) Q_{it}^{-1}  \left( g_{it}(\theta_{it}) - g_{it}(\bar \theta_{it})\right)\\
    &= \frac{1}{C_\mu^2} \left\| \left( g_{it}(\theta_{it}) - g_{it}(\bar \theta_{it})\right)\right\|_{Q_{it}^{-1}},
\end{align*}
where the first equality follows from Eq~\eqref{eq:lem-1-eq-1}, and the inequality follows from Eq~\eqref{eq:lem-1-eq-2}.

Therefore, 
\begin{align*}
    \|\theta_{i}^\ast - \bar \theta_{it}\|_{Q_{it}} &\leq \frac{1}{C_\mu} \left\| \left( g_{it}(\theta_{it}) - g_{it}(\bar \theta_{it})\right)\right\|_{Q_{it}^{-1}} \\
    & \leq \frac{2}{C_\mu} \left\| \left( g_{it}(\theta_{it}) - g_{it}(\bar \theta_{it}^{MLE})\right)\right\|_{Q_{it}^{-1}}\\
    &= \frac{2}{C_\mu} \left\|\sum_{s=1}^{t-1} \phi(\vx_{is})(Y_{is} - \mu\left( \phi(\vx_{is})^\top \theta_i^\ast\right))\right\|_{Q_{it}^{-1}}.
\end{align*}
where the second inequality is from the triangle inequality, and the last equality is from the definition of $\theta_{it}^{MLE}$ and $g_{it}$. 

Let $A_{it}$ denote the event that 
\[
\left\|\sum_{s=1}^{t-1} \phi(\vx_{is})(Y_{is} - \mu\left( \phi(\vx_{is})^\top \theta_i^\ast\right))\right\|_{Q_{it}^{-1}} \leq \kappa \sigma \sqrt{2m\log(t)} \sqrt{\log\left(\frac{d}{\delta_{t}}\right)},
\]
then we have $A_{it}$ holds with probability at least $\delta_{t}$ by Lemma~\ref{lem:TL5}.

\end{proof}

\begin{lemma}\label{lem:lemma2}
    For any round $t > t_0$, $P_{t} (B_{it}) \geq 1 - \delta_{2t}$. 
\end{lemma}

\begin{proof}
First, recall that $B_{it} = \{ \| \bar \theta_{it} - \theta_{it}  \|_{Q_{it}}  \leq \beta_{2t}\}$. We can now write,
\begin{eqnarray*}
   P(B_{it}^c) &=& P(\| \theta_{it} - \bar \theta_{it} \|_{Q_{it}} > \beta_{2t}) \\\\
  & \leq &P(\| \theta’_{it} - \bar \theta_{it}\|_{Q_{it}} > \beta_{2t}) \\\\
   & = &P(\|\theta’_{it} - \bar \theta_{it}\|_{\alpha_t^{-1} Q_{it}} > \alpha_t^{-\frac{1}{2}} \beta_{2t} ) \\\\
   & =&  P(\sqrt{Z} > \sqrt{M_t \gamma_{2t}}),
\end{eqnarray*}

where $Z = \left( \theta_{it} - \bar \theta_{it}\right)^\top \alpha_t^{-2}Q_{it} \left( \theta_{it} - \bar \theta_{it}\right)$. The first step simply uses the fact that since $\bar{\theta}_{it}$ is already inside $\Theta$ (see line 17 in Algorithm 1), projecting $\theta’_{it}$ to be inside $\Theta$ after sampling only brings it \textit{even closer} to $\bar{\theta}_{it}$.


Note that $Z$ is a $\chi^2_m$ random variable. This follows from the fact that 
\[
\theta_{it} \sim \cN(\bar \theta_{it}, \alpha_t^2  Q_{it}^{-1}),
\]
therefore we have
\[
\alpha_t^{-1} Q_{it}^{1/2}  \left(\theta_{it} - \bar \theta_{it}\right) \sim \cN(0, I_m).
\]
Denote $y =\alpha_t^{-1} Q_{it}^{1/2} \left(\theta_{it} - \bar \theta_{it}\right)$, then $Z =  y^\top y$ is a $\chi_m^2$ random variable.

Therefore, by Lemma~\ref{lem:TL1},  and the definition that $\gamma_{2t} = \max\left(8\log(\frac{1}{\delta_{2t}}), \sqrt{\log(\frac{1}{\delta_{2t}})}\right)$, we have
\[
P(B_{it}^c) = P(Z > n + \gamma_{2t}) \leq \delta_{2t},
\]
which completes the proof.
\end{proof}

\begin{lemma} \label{lem:lemma3}
    Let $\vx$ be arbitrary such that $\vx \in \cX$. Then, 
    \[
    P_t(u_{it}(\vx) > u_i(x) | A_{it} ) \geq q_0,
    \]
    with $q_0 = \sqrt{\frac{2}{\pi}}\frac{1}{\sqrt{2}+\sqrt{6}}\frac{1}{e} \approx 0.075$.
\end{lemma}
\begin{proof}
First, notice that
\begin{align*}
    u_{it}(\vx) > u_i(\vx) & \iff \mu \left( \theta_{it}^\top \phi(\vx)\right) > \mu \left( (\theta_i^\ast)^\top \phi(\vx)\right)\\
    & \iff \theta_{it}^\top \phi(\vx) >  (\theta_i^\ast)^\top \phi(\vx)\\
    & \iff \frac{\left( \theta_{it} - \bar \theta_{it}\right)^\top \phi(\vx)}{\alpha_t \rho_{it}(\vx)} > \frac{\left( \theta_{i}^\ast- \bar \theta_{it}\right)^\top \phi(\vx)}{\alpha_t \rho_{it}(\vx)}.
\end{align*}
Since $\theta_{it} \sim \cN (\bar \theta_{it}, \alpha_t^2 Q_{it}^{-1})$, we have
\begin{align*}
\left( \theta_{it} - \bar \theta_{it}\right)^\top \phi(\vx) \sim \cN(0, \alpha_t^2 \phi(x)^\top Q_{it}^{-1} \phi(x)),
&\;\implies\;
\left( \theta_{it} - \bar \theta_{it}\right)^\top \phi(\vx) \sim \cN(0, \alpha_t^2 \rho_{it}^2(\vx)).
\\
&\;\implies\;
\frac{\left( \theta_{it} - \bar \theta_{it}\right)^\top \phi(\vx)}{\alpha_t \rho_{it}(\vx)} \sim \cN(0,1).
\end{align*}
From the above we have that
\[
P_t\left(u_{it}(\vx) > u_i(\vx) | A_{it} \right) = P_t \left( Z > \frac{\left( \theta_{it} - \bar
\theta_{it}\right)^\top \phi(\vx)}{\alpha_t \rho_{it}(\vx)} \bigg| A_{it}\right),
\]
where $Z \sim \cN(0,1)$ is sampled independently of the observations, since the randomness in Algorithm~\ref{Alg:TS} can be assumed to be independent of the randomness in the observations. Therefore, under the event $A_{it}$,
\begin{align*}
  \left| \frac{\left( \theta_{it} - \bar \theta_{it}\right)^\top \phi(\vx)}{\alpha_t \rho_{it}(\vx)}\right|  &= \left| \frac{\left( \theta_{it} - \bar \theta_{it}\right)^\top Q_{it}^{\frac{1}{2}}Q_{it}^{-\frac{1}{2}}\phi(\vx)}{\alpha_t \rho_{it}(\vx)} \right|\\
  &\leq \frac{\| \theta_{it} - \bar \theta_{it}\|_{Q_{it}} \|\phi(\vx)\|_{Q_{it}^{-1}}}{\alpha_t \rho_{it}(\vx)}\\
  &\leq \frac{\beta_{it}}{\alpha_t} = \frac{\sqrt{8}}{\alpha_0}.
\end{align*}
Here, the first inequality follows from the definition of the matrix norm and the definition of $A_{it}$, and the second inequality follows from the definition of $\beta_{1t}$.

Therefore, by Lemma~\ref{lem:TL2}, we have
\begingroup
\begin{align*}
    &P_t(u_{it}(\vx) > u_i(x) | A_{it} ) \\
    &= P_{Z \sim \cN(0,1)} (Z > \frac{\sqrt{8}}{\alpha_0}) \\
    &\geq \sqrt{\frac{2}{\pi}} \frac{1}{\sqrt{\nicefrac{8}{\alpha_0^2}} + \sqrt{4+\sqrt{\nicefrac{8}{\alpha_0^2}}}} e^{-\frac{4}{\alpha_0^2}}.
\end{align*}
\endgroup
Setting $\alpha_0^2 = 4$, we have
\[
 P_t\left(u_{it}(\vx) > u_i(x) | A_{it} \right) \geq \sqrt{\frac{2}{\pi}}\frac{1}{\sqrt{2}+\sqrt{6}}\frac{1}{e} \approx 0.075,
\]
which completes the proof.
\end{proof}

\begin{lemma}\label{lem:lemma4}
Let $\theta_1, \theta_2 \in \Theta \in \R^m$. Let $Q \succeq 0, Q \in \R^{m \times m}$ be a positive semi-definite matrix, and $\rho_Q(\vx) = \sqrt{\phi(\vx)^\top Q^{-1} \phi(\vx)}$. Then, 
\[
\left|\mu(\theta_1^\top \phi(\vx)) - \mu(\theta_2^\top \phi(\vx))\right| \leq L_{\mu} \|\theta_1 - \theta_2\|_{Q} \cdot \rho_{Q}(\vx).
\]
\end{lemma}
\begin{proof}
This follows from the Lipschitz properties of $\mu$ and the following simple calculations:
\begingroup
\allowdisplaybreaks
\begin{align*}
    \left|\mu(\theta_1^\top \phi(\vx)) - \mu(\theta_2^\top \phi(\vx))\right| &\leq L_\mu |(\theta_1 - \theta_2)^\top \phi(\vx)|\\
    &= L_\mu |(\theta_1- \theta_2)^\top Q^{\frac{1}{2}} Q^{-\frac{1}{2}} \phi(\vx)| \\
    &\leq L_\mu \|\theta_1 - \theta_2\|_{Q} \|\phi(\vx)\|_{Q^{-1}}\\
    &=  L_{\mu} \|\theta_1 - \theta_2\|_{Q} \cdot \rho_{Q}(\vx).
\end{align*}
\endgroup
\end{proof}

\begin{lemma}\label{lem:lemma5}
For any round $t > t_0$, $P_{t} (\vx_{it} \notin \cS_{it}) \geq q_0 (1 - \delta_{t}) - \delta_{2t}$. 
\end{lemma}

\begin{proof}
First, when event $B_{it}$ holds, by lemma~\ref{lem:lemma4}, we have that for all $\vx$, 
\[
|u_{it}(\vx) - \bar u_{it}(\vx)| \leq L_\mu \beta_{2t} \rho_{it}(\vx).
\]
Note that by definition, $u_{it}(\vx)  = \mu\left((\theta_{it})^\top \phi(\vx)\right)$, and $\bar u_{it}(\vx)  = \mu\left((\bar \theta_{it})^\top \phi(\vx)\right)$. Therefore, 
\begin{equation}\label{eq:lem5-eq1}
    \bar u_{it}(\vx) - u_{it}(\vx)  > -L_\mu \beta_{2t}  \rho_{it}(\vx).
\end{equation}

On the other side, under event $A_{it}$, by lemma~\ref{lem:lemma4}, we have that for all $\vx$, 
\begin{equation}\label{eq:lem5-eq2}
    | u_{i}(\vx) - \bar u_{it}(\vx)| \leq L_\mu \beta_{1t} \rho_{it}(\vx).
\end{equation}

Moreover, recall that by definition for any $\vx \in \cS_{it}$, 
\begin{equation}\label{eq:lem5-eq3}
    u_i( \vx_{it}^\ast) - u_i(\vx) \geq \beta_{3t} \rho_{it} (\vx).
\end{equation}

Therefore, consider any $\vx \in \cS_{it}$, and under the condition that $A_{it}\cap B_{it} \cap  \{u_{it}( \vx_{it}^\ast) > u_i(\vx_{it}^\ast)\}$, we have
\begin{align}\label{eq:lem5-eq-num1}
\begin{split}
u_{it}( \vx_{it}^\ast) - u_{it}(\vx) &> u_i( \vx_{it}^\ast) - u_{it}(\vx) \\
&=   \left(u_i( \vx_{it}^\ast) - u_i(\vx) \right) + \left(  u_{i}(\vx) - \bar u_{it}(\vx)  \right) + \left( \bar u_{it}(\vx) - u_{it}(\vx)  \right) \\
&>0,
\end{split}
\end{align}
where the last inequality follows from combining equations Eq~\eqref{eq:lem5-eq1}, Eq~\eqref{eq:lem5-eq2},  Eq~\eqref{eq:lem5-eq3} and the definition of $\beta_{3t}$. Hence, Eq~\eqref{eq:lem5-eq-num1} implies that, under the same condition, $\vx_{it} \notin \cS_{it}$ since by construction, $\vx_{it}$ maximizes $u_{it}$ under the budget, thus 
\[
u_{it}(\vx_{it}) \geq u_{it}(\vx_{it}^\ast),
\]
This further implies that,
\begingroup
\allowdisplaybreaks
\begin{align*}
    P_t({x_{it} \notin \cS_{it}}) &\geq P_t\left (u_{it}(\vx_{it}^\ast) > u_{it}(x), \forall x \in \cS_{it}\right)\\
    &\geq P_t\left (u_{it}(\vx_{it}^\ast) > u_{it}(x), \forall x \in \cS_{it} | A_{it}\cap B_{it} \{u_i( \vx_{it}^\ast) > u_i(\vx)\} \right)\\
    & \hspace{0.5in} \times P(A_{it}\cap B_{it} \cap  \{u_i( \vx_{it}^\ast) > u_i(\vx)\})\\
    &= P(A_{it}\cap B_{it} \cap \{u_i( \vx_{it}^\ast) > u_i(\vx)\}) \\
    &\geq P(A_{it}\cap \{u_i( \vx_{it}^\ast) > u_i(\vx)\}) - P( B^c_{it})\\
    &= P(\{u_i( \vx_{it}^\ast) > u_i(\vx)\}|A_{it})P(A_{it}) - P( B^c_{it})\\
    &\geq q_0 (1-\delta_{t}) + \delta_{2t}.
\end{align*}
\endgroup
Here, the second and third inequality both from the law of total probability and rearranging terms, and the last inequality follows from Lemma~\ref{lem:lemma1}, Lemma~\ref{lem:lemma2} and Lemma~\ref{lem:lemma3}, which completes the proof.
\end{proof}

\begin{lemma}\label{lem:lemma6}
For $t \geq \max(t_0, t_0')$, 
\[
\mathbb{E}_t[\ell_{it}] \leq \frac{5}{q_0}\beta_{3t} \mathbb{E}_t[\rho_{it}(\vx_{it})] + \delta_{t} + \delta_{2t}.
\]
Here, $\ell_{it}= \left( u_i(\vx_{it}^\ast) -  u_i(\vx_{it})\right)^{+}$, and $t_0'$ is chosen such that, $\forall t > t_0'$, $\delta_{t} < \frac{1}{4}$, $\delta_{2t} < \frac{q_0}{4}$.
\end{lemma}

\begin{proof}
First, define 
\[
\vx_{it}' = \argmin_{\vx: p_{t}^\top \leq p_t^\top e_i, \vx \notin \cS_{it}} \rho_{it}(\vx).
\] 
This implies that,
\[
\E_t[\rho_{it}(\vx_{it})] \geq \E_t[\rho_{it}(\vx_{it})|\vx_{it} \notin \cS_{it}] P(\vx \notin \cS_{it}) \geq \rho_{it} (\vx_{it}') P(\vx \notin \cS_{it}).
\]
Therefore, by Lemma~\ref{lem:lemma5}, we have
\[
\rho_{it} (\vx_{it}') \leq \frac{\E_t[\rho_{it}(\vx_{it})]}{q_0(1-\delta_{t}) - \delta_{2t}}.
\]
Select $t_0'$ such that, $\forall t \geq t_0'$, $\delta_{t} = \frac{1}{4}$, and $\delta_{2t} \leq \frac{q_0}{4}$, then we have:
\[
\rho_{it}(\vx_{it}') \leq \frac{2}{q_0} \E_t[\rho_{it}(\vx_{it})].
\]
Also, under $A_{it} \cap B_{it}$,
\[
\|\theta_{i}^\ast - \theta_{it}\|_{Q_{it}} \leq \|\theta_i^\ast - \bar \theta_{it}\|_{Q_{it}} + \|\bar \theta_{it} - \theta_{it}\|_{Q_{it}} \leq \beta_{1t} + \beta_{2t},
\]
where the first inequality follows from triangle inequality, and the second one follows from the definitions of $A_{it}$ and $B_{it}$. Hence,
\[
|\mu_{it}(\vx) - \mu_i(\vx) \leq L_\mu (\beta_{1t} + \beta_{2t})\rho_{it}(\vx) = \beta_{3t} \rho_{it}(\vx).
\]
Therefore, we have 
\begin{align*}
    \ell_{it} &= u_i(\vx_{it}^\ast) - u_i(\vx_{it}) \\
    &= u_i(\vx_{it}^\ast) - u_i(\vx_{it}') +  u_i(\vx_{it}')- u_i(\vx_{it})\\
    &\leq2\beta_{3t}\rho_{it}(\vx_{it}') + \beta_{3t}\rho_{it}(\vx_{it})\\
    &\leq \frac{4}{q_0} \beta_{3t} \E[\rho_{it}(\vx_{it})] + \beta_{3t} \rho_{it}(\vx_{it}).
\end{align*}
which further yields
\begin{align}\label{lCE-bound}
\begin{split}
    \mathbb{E}_t[\ell_{it}] & \leq  \E_t[\ell_{it}| A_{it} \cap B_{it}] +  \E_t[\ell_{it}| A_{it}^c \cup B_{it}^c] P(A_{it}^c \cup B_{it}^c) \\
    &\leq \frac{4}{q_0} \beta_{3t} \E[\rho_{it}(\vx_{it})] + \beta_{3t}\E_t[\rho_{it}(\vx_{it})] + \delta_{t} + \delta_{2t}\\
    &\leq \frac{5}{q_0}\beta_{3t} \mathbb{E}_t[\rho_{it}(\vx_{it})] + \delta_{t} + \delta_{2t},
\end{split}
\end{align}
which completes the proof.
\end{proof}

\begin{lemma}\label{lem:lemma7}
Let $\delta' > 0$. Define $L_{iT} = \sum_{t=1}^T \ell_{it}$. Then, with probability at least $1-\delta'$, 
\[
L_{iT}  \leq \sum_{t=1}^T (\delta_{t} + \delta_{2t}) + \tilde{O} \bigg( \frac{m^2}{\sqrt{M}} \sqrt{T} \left( \log(T) + \log \left(\frac{1}{\delta'}\right) \right) \bigg)
\]
\end{lemma}
\begin{proof}
First, define for $s > 1$, 
\[
u_{is} = \ell_{is} - \frac{5\beta_{3s}}{q_0} \rho_{is}(\vx_{is}) - (\delta_{s} + \delta_{2s}),
\]
and $v_{it} = \sum_{s=1}^t u_{is}$, with $v_{i0} = 0$ and $u_{i0} = 0$. We show that $\{v_{it}\}, t
\geq 0$ is a super-martingale with respect to the filtration $(\cF_t)_{t \geq 0}$.

First, 
\[
\E_t[u_{it}] = \E_t[\ell_{it}] - \frac{5\beta_{3t}}{q_0}\E_t[\rho_{it}(\vx_{it})] - (\delta_{t} + \delta_{2t}) \leq 0.
\]
Moreover, 
\begin{align*}
    |v_{it} - v_{i, t-1}| &\leq |\ell_{it}| + \frac{5\beta_{3t}}{q_0} |\rho_{it}(\vx_{it})| + (\delta_{t} + \delta_{2t})\\
    &\leq 1 + \frac{5\beta_{3t}}{q_0} \frac{\|\phi(1)\|_2}{\sqrt{M}} + 1 \\
    &\leq \frac{7\beta_{3t}}{q_0} \frac{\|\phi(\vone)\|_2}{\sqrt{M}} \triangleq D_t.
\end{align*}
Therefore, by Lemma~\ref{lem:TL3}, with probability at least $1 - \delta'$, we have that
\[
v_{iT} - v_{i,0} \leq \sqrt{s\log \left(\frac{1}{\delta'}\sum_{t-1}^T D_t^2 \right)} \leq \frac{7\beta_{3t}}{q_0} \|\phi(1)\|_2 \sqrt{\frac{2\log(\frac{1}{\delta'})}{M}T}.
\]
Therefore, we have
\begin{equation}\label{eq:lem7-eq1}
    L_{iT} = \sum_{s=1}^T \ell_{is} \leq \sum_{t=1}^T (\delta_{t} + \delta_{2t}) + \frac{5\beta_{3t}}{q_0}\sum_{t=1}^T \rho_{it}(\vx_t) +  \frac{7\beta_{3t}}{q_0} \|\phi(\vone)\|_2 \sqrt{\frac{2\log(\frac{1}{\delta'})}{M}T}.
\end{equation}

Now it remains to bound $\sum_{t=1}^T \rho_{it}(\vx_t)$. Since $Q_{it} \succeq M I$, by the definition of $\rho_{it}(\vx_{it})$, we have
\[
\rho_{it}x_{it} \leq \frac{\|\phi(\vone)\|_2}{\sqrt{M}}.
\]
Hence, by Lemma~\ref{lem:TL4} and rearranging terms, we have
\begin{equation}\label{eq:lem7-eq-last}
    \sum_{t=t_0}^T \rho_{it}^2(\vx_{it}) \leq \frac{\|\phi(\vone)\|^2_2}{M} \frac{1}{\log \left(1 + \frac{\|\phi(\vone)\|_2}{\sqrt{M}} \right)} \sum_{t=t_0}^T \log \left(1 + \phi^\top(\vx_{it}) Q_{it}^{-1} \phi(\vx_{it}) \right).
\end{equation}

Also notice that, 
\[
\sum_{t=t_0}^T \log \left(1 + \phi^\top(\vx_{it}) Q_{it}^{-1} \phi(\vx_{it})\right) = \log \Pi_{t=t_0}^T \left(1 + \|\phi(\vx_{it})\|^2_{Q_{it}^{-1}} \right) = \log \frac{\det(Q_{iT})}{\det(Q_{i,t_0})}.
\]

Note that the trace of $Q_{i,t+1}$ is upper-bounded by $t\cdot \|\phi(\vone)\|_2$, then given that the trace of the positive definite
matrix $Q_{iT}$ is equal to the sum of its eigenvalues, we have that $\det(Q_{iT}) \leq \left (t\left \|\phi(\vone)\right \|_2^2\right)^m$. Moreover, $\det(Q_{i,t_0}) \geq (M)^m$, therefore, 
\[
\sum_{t=t_0}^T \log \left(1 + \phi^\top(\vx_{it}) Q_{it}^{-1} \phi(\vx_{it})\right) \leq m \log \left(\frac{\|\phi(\vone)\|_2^2T}{M}\right).
\]

Combining with Eq~\eqref{eq:lem7-eq-last}, and applying Cauchy-Schwartz inequality, we have
\begin{equation}\label{eq:lem7-eq2}
    \sum_{t=t_0}^T \rho_{it}(\vx_{it}) \leq \sqrt{T \sum_{t=t_0}^T \rho_{it}^2(\vx_{it})} \leq \sqrt{T}\sqrt{\frac{\|\phi(\vone)\|^2_2}{M} \frac{m}{\log \left(1 + \frac{\|\phi(\vone)\|_2}{\sqrt{M}} \right)}   \log \left(\frac{\|\phi(\vone)\|_2^2T}{M}\right) }.
\end{equation}

Putting together Eq~\eqref{eq:lem7-eq1} and Eq~\eqref{eq:lem7-eq2}, with the fact that $\|\phi(\vone)\|_2 = O(\sqrt{m})$, e.g. $\|\phi(\vone)\|_2 \leq c_{\phi}\sqrt{m}$,  we have that with probability at least $1-\delta'$,
\begin{align*}
L_{iT} &\leq \sum_{t=1}^T (\delta_{t} + \delta_{2t}) + \frac{\beta_{3t}}{q_0}\frac{\sqrt{T}}{\sqrt{M}}\left( 5c_{\phi}m \sqrt{ \frac{1}{\log \left(1 + \frac{\|\phi(\vone)\|_2}{m} \right)}  \log \left(\frac{\|\phi(\vone)\|_2^2T}{M}\right) }  + 7c_{\phi} \sqrt{m} \sqrt{2\log(\frac{1}{\delta'})} \right)\\
&= \sum_{t=1}^T (\delta_{t} + \delta_{2t}) + \tilde{O} \bigg(  \frac{m^2}{\sqrt{M}}\sqrt{T} \left( \log(T) + \log \left(\frac{1}{\delta'}\right) \right) \bigg),
\end{align*}
where the last step comes from the fact that $\beta_{3t} = \tilde O(m)$. This completes the proof.
\end{proof}

\subsection{Proof of Theorem~\ref{thm:TS-anytime} for $\LCE$}\label{app:proof-TS-anytime-LCE}



\begin{corollary}\label{corr:CE-loss} With probability at least $1-\delta'$, 
\[
\LCE_T  \leq n\sum_{t=1}^T (\delta_{1t} + \delta_{2t}) + \tilde{O} \left( n\frac{m^2}{\sqrt{M}} \sqrt{T} \left( \log(T) + \log (\frac{1}{\delta'})\right) \right)
\]
\end{corollary}
\begin{proof}
This a direct result from Lemma~\ref{lem:lemma7} and the definition of $\LCE_T$: With probability at least $1-\delta'$, 
\begin{align*}
    \LCE_T &= \sum_{i=1}^n \sum_{t=1}^T \left(\max_{\vy: p(\vy) \leq p(\ve_i)} u_i(\vy) - u_i(\vx_{it})\right)^{+} \\
    &\leq n\sum_{t=1}^T (\delta_{t} + \delta_{2t}) + \tilde{O} \left( n\frac{m^2}{\sqrt{M}}\sqrt{T} \left( \log(T) + \log \left(\frac{1}{\delta'}\right)\right) \right).
\end{align*}
\end{proof}


\begin{proof}
Choose $\delta_{t} =\delta_{2t} = \frac{2\delta}{n\pi^2 t^2}$. Then,
\[
\sum_{t=1}^T (\delta_{t} + \delta_{2t}) \leq \frac{2}{3}\delta.
\]
Also choose $\delta' = \frac{\delta}{3}$, then by Corollary~\ref{corr:CE-loss}, with probability at least $1-\delta$,
\[
 \LCE_T = O\left(n\frac{m^2}{\sqrt{M}}\sqrt{T}\left(\log\left(\frac{\delta}{3}\right) + \log(T))\right)\right).
\]


\end{proof}

\subsection{Proof of Theorem~\ref{thm:TS-finitehorizon} for $\LCE$}\label{app:proof-TS-finite-time-CE}


\begin{proof}
Choose $\delta_{1t} =\delta_{2t} = \delta' = \frac{1}{T}$. Denote the event where $ \LCE_T = O\left(n\frac{m^2}{\sqrt{M}}\sqrt{T}\left(\log\left(\frac{\delta}{3}\right) + \log(T))\right)\right)$ holds as $\mathcal E$.

Then, by Lemma~\ref{lem:lemma7},
\[
\E[\LCE_T] = \E[\LCE_T|\cE] + \E[\LCE_T|\cE^c] P(\cE^c) \leq 2 + O\left(n\frac{m^2}{\sqrt{M}}\sqrt{T}\left(\log\left(\frac{\delta}{3}\right) + \log(T))\right)\right).
\]

where $M \geq m$. This completes the proof.
\end{proof}

\section{Bounding $\LFD$}\label{sec:proof-TS-anytime-LFD}

Recall that the definition of $\lFD$ is directly based on the requirements of Pareto efficiency and fair share: $\lFD(x) \defeq \max \left(\lPE(x), \lSI(x)\right)$, where $\lPE(x) = \min_{x' \in \PE} \sum_{i=1}^n \left(u_i(\vx_i') - u_i(\vx_i)\right)^{+}$; and $\lSI(x) = \sum_{i=1}^n \left(u_i(\ve_i) - u_i(\vx_i)\right)^{+}$. 

To bound $\LFD$, we first provide a useful lemma which shows that $\lSI$ is a weaker notion than $\lCE$.
\begin{lemma}\label{lem:lSI-lCE}
For any allocation $x$ and price $p$,  $\lSI(x) \leq \lCE (x, p)$.
\end{lemma}
\begin{proof}
This simply uses the fact that an agent's endowment is always affordable under any price vector $p$.
Therefore,
\begingroup
\allowdisplaybreaks
\begin{align*}
     \lSI(x) &=\sum_{i=1}^n \left(u_i(e_i) - u_i(\vx_i)\right)^{+} \\
     & \leq  \sum_{i=1}^n \left(\max_{y: p\top y \leq p^\top \ve_i} u_i(y) - u_i(x)\right)^{+}, \forall p \\
     &= \lCE (x, p),
\end{align*}
\endgroup
\end{proof}

Having lemma~\ref{lem:lSI-lCE} at hand, the key remaining task is to bound $\lPE$. We will show that
this can be achieved by an analogous analysis as in Section~\ref{app:proof-TS-anytime-LCE}, but with some key differences.

First, we define $\tilde \cS_{it}$ (in comparison to $\cS_{it}$ used in Section~\ref{app:proof-TS-anytime-LCE}):
\[
\tilde \cS_{it} \defeq \left\{\vx \in \cX: u_i( \vx_{i}^\ast) - u_i(\vx) \geq \beta_{3t} \rho_{it} (\vx) \right\},\]
where $x^\ast \in \R^{n\times m}$ is the unique equilibrium allocation. Note that $\tilde \cS_{it}$ shares a similar spirit as $\cS_{it}$, which is used in Section~\ref{app:proof-TS-anytime-LCE}, but with a different referencing point $x^\ast$.

We show a key lemma which provides a lower bound on $P(x \not \in \tilde \cS_{it})$.

\begin{lemma}\label{lem:lemma5-FD}
For any round $t > t_0$, $P_{t} (\vx_{it} \notin \tilde \cS_{it}) \geq q_0 (1 - \delta_{t}) - \delta_{2t}$. 
\end{lemma}

\begin{proof}
First, when event $B_{it}$ holds, by lemma~\ref{lem:lemma4}, we have that for all $\vx$, 
\[
|u_{it}(\vx) - \bar u_{it}(\vx)| \leq L_\mu \beta_{2t} \rho_{it}(\vx).
\]
Note that by definition, $u_{it}(\vx)  = \mu\left((\theta_{it})^\top \phi(\vx)\right)$, and $\bar u_{it}(\vx)  = \mu\left((\bar \theta_{it})^\top \phi(\vx)\right)$. Therefore, 
\begin{equation}\label{eq:lem5-eq1-FD}
    \bar u_{it}(\vx) - u_{it}(\vx)  > -L_\mu \beta_{2t}  \rho_{it}(\vx).
\end{equation}

On the other side, under event $A_{it}$, by lemma~\ref{lem:lemma4}, we have that for all $\vx$, 
\begin{equation}\label{eq:lem5-eq2-FD}
    | u_{i}(\vx) - \bar u_{it}(\vx)| \leq L_\mu \beta_{1t} \rho_{it}(\vx).
\end{equation}

Moreover, recall that by definition for any $\vx \in \tilde \cS_{it}$, 
\begin{equation}\label{eq:lem5-eq3-FD}
    u_i( \vx_{i}^\ast) - u_i(\vx) \geq \beta_{3t} \rho_{it} (\vx).
\end{equation}

Therefore, consider any $\vx \in \tilde \cS_{it}$, and under the condition that $A_{it}\cap B_{it} \cap  \{u_{it}( \vx_{i}^\ast) > u_i(\vx^\ast_i)\}$, we have
\begin{align}\label{eq:lem5-eq-num1-FD}
\begin{split}
u_{it}( \vx_{i}^\ast) - u_{it}(\vx) &> u_i( \vx_{i}^\ast) - u_{it}(\vx) \\
&=   \left(u_i( \vx_{i}^\ast) - u_{i}(\vx) \right) + \left(  u_{i}(\vx) - \bar u_{it}(\vx)  \right) + \left( \bar u_{it}(\vx) - u_{it}(\vx)  \right) \\
&>0,
\end{split}
\end{align}
where the last inequality follows from combining equations Eq~\eqref{eq:lem5-eq1-FD}, Eq~\eqref{eq:lem5-eq2-FD},  Eq~\eqref{eq:lem5-eq3-FD} and the definition of $\beta_{3t}$. Moreover, recall that $\vx_{it}$ maximizes $u_{it}$ under the budget, thus 
\[
u_{it}(\vx_{it}) \geq u_{it}(\vx_i^\ast),
\]
Therefore, Eq~\eqref{eq:lem5-eq-num1-FD} implies that, $\vx_{it} \notin \tilde \cS_{it}$.
This further implies,
\begin{align*}
    P_t({x_{it} \notin \tilde \cS_{it}}) &\geq P_t\left (u_{it}(\vx_{i}^\ast) > u_{it}(x), \forall x \in \tilde \cS_{it}\right)\\
    &\geq P_t\left (u_{it}(\vx_{i}^\ast) > u_{it}(x), \forall x \in \tilde \cS_{it} | A_{it}\cap B_{it} \{u_i( \vx_{it}^\ast) > u_i(\vx)\} \right)\\
    & \quad \cdot P(A_{it}\cap B_{it} \cap  \{u_i( \vx_{it}^\ast) > u_i(\vx)\})\\
    &= P(A_{it}\cap B_{it} \cap \{u_i( \vx_{i}^\ast) > u_i(\vx)\}) \\
    &\geq P(A_{it}\cap \{u_i( \vx_{i}^\ast) > u_i(\vx)\}) - P( B^c_{it})\\
    &= P(\{u_i( \vx_{i}^\ast) > u_i(\vx)\}|A_{it})P(A_{it}) - P( B^c_{it})\\
    &\geq q_0 (1-\delta_{t}) + \delta_{2t}.
\end{align*}
Here, the second and third inequality both from the law of total probability and rearranging terms, and the last inequality follows from Lemma~\ref{lem:lemma1}, Lemma~\ref{lem:lemma2} and Lemma~\ref{lem:lemma3}, which completes the proof.
\end{proof}

\begin{lemma}\label{lem:lPE-ub}
At any round $t>t_0$, define $\vx''_{it} \defeq \argmin_{x \not \in \tilde \cS_{it}, \vp_t^\top \vy \leq \vp_t^\top \ve_i} \rho_{it}(x_{it})$, then we have
\[
\lPE(x_t) \leq  \sum_{i \in [n]} \bPE_{it},
\]
 with probability at least $1 - \delta_t - \delta_{2t}$, and $\bPE_{it} = 2\beta_{3t}\rho_{it}(x''_{it}) + \beta_{3t}\rho_{it}(x_{it})$.
\end{lemma}
\begin{proof}
Begin with the definition of $\lPE$, we have
\begin{align*}
    \lPE(x_t) &= \min_{x \in \PE} \sum_{i\in[n]} \left(u_i(x_i) - u_i(x_{it})\right)\\
    &\leq  \sum_{i\in[n]} \left(u_i(x_i^\ast) - u_i(x_{it})\right)\\
    &=  \sum_{i\in[n]} \left(u_i(x_i^\ast) - u_i(x''_{it})\right) +  \sum_{i\in[n]} \left(u_i(x''_i) - u_i(x_{it})\right) \\
    &\leq \beta_{3t} \sum_{i\in[n]} \rho_{it}(x''_{it}) + \sum_{i\in[n]} \left(u_i(x''_i) - u_i(x_{it})\right),
\end{align*}
where the last inequality follows from this definition. Moreover, we have
\[
u_i(x''_{it}) \leq u_{it}(x''_{it}) + \beta_{3t}\rho_{it}(x''_{it}),
\]
and 
\[
u_i(x_{it}) \geq u_{it}(x_{it}) - \beta_{3t}\rho_{it}(x_{it}).
\]
under the event $A_{it} \cap B_{it}$, by Eq~\eqref{eq:lem5-eq1-FD} and Eq~\eqref{eq:lem5-eq2-FD}. Putting these together yields
\[
 \lPE(x_t) \leq 2\beta_{3t}\sum_{i\in[n]}\rho_{it}(x''_{it}) + \beta_{3t}\sum_{i\in [n]}\rho_{it}(x_{it}) \defeq \sum_{i \in [n]} \bPE_{it},
\]
which completes the proof.
\end{proof}

Now we show that the above result leads to the lemma below, which shows a analogous guarantee as we obtained in lemma~\ref{lem:lemma6}.

\begin{lemma}\label{lem:lemma6-FD} 
For $t \geq \max(t_0, t_0')$, 
\[
\mathbb{E}_t[\sum_{i\in [n]}\lFD_{it}] \leq \frac{5}{q_0}\beta_{3t} \mathbb{E}_t[\sum_{i\in [n]} \rho_{it}(\vx_{it})] + n(\delta_{t} + \delta_{2t}).
\]
Here, $t_0'$ is chosen such that, $\forall t > t_0'$, $\delta_{t} < \frac{1}{4}$, $\delta_{2t} < \frac{q_0}{4}$.
\end{lemma}
\begin{proof}
First, note that, by the definition of $\vx_{it}''$,
\[
\E_t[\rho_t(\vx_{it})] \geq \E_t[\rho_t(\vx_{it})|\vx_{it} \notin \cS_{it}] P(\vx \notin \cS_{it}) \geq \rho_{it} (\vx_{it}'') P(\vx \notin \cS_{it}).
\]
Moreover, combining the above with Lemma~\ref{lem:lemma5-FD}, we have
\[
\rho_{it} (\vx_{it}'') \leq \frac{\E_t[\rho_{it}(\vx_{it})]}{q_0(1-\delta_{t}) - \delta_{2t}}.
\]
Select $t_0'$ such that, $\forall t \geq t_0'$, $\delta_{t} = \frac{1}{4}$, and $\delta_{2t} \leq \frac{q_0}{4}$, then we have:
\[
\rho_{it}(\vx_{it}'') \leq \frac{2}{q_0} \E_t[\rho_{it}(\vx_{it})].
\]
Also, under $A_{it} \cap B_{it}$,
\[
\|\theta_{i}^\ast - \theta_{it}\|_{Q_{it}} \leq \|\theta_i^\ast - \bar \theta_{it}\|_{Q_{it}} + \|\bar \theta_{it} - \theta_{it}\|_{Q_{it}} \leq \beta_{1t} + \beta_{2t},
\]
where the first inequality follows from triangle inequality, and the second one follows from the definitions of $A_{it}$ and $B_{it}$. Hence,
\[
|\mu_{it}(\vx) - \mu_i(\vx) \leq L_\mu (\beta_{1t} + \beta_{2t})\rho_{it}(\vx) = \beta_{3t} \rho_{it}(\vx).
\]
Therefore, we have 
\begin{align*}
    \sum_{i\in [n]}\bPE_{it} &\leq2\beta_{3t}\sum_{i\in [n]}\rho_{it}(\vx_{it}'') + \beta_{3t}\sum_{i\in [n]}\rho_{it}(\vx_{it})\\
    &\leq \frac{4}{q_0} \beta_{3t} \E[\sum_{i\in [n]}\rho_{it}(\vx_{it})] + \beta_{3t} \sum_{i\in [n]}\rho_{it}(\vx_{it}).
\end{align*}

Moreover, by Lemma~\ref{lem:lSI-lCE} and eq~\eqref{lCE-bound} which holds under the same condition of $A_{it} \cap B_{it}$, we have
\begin{align*}
\sum_{i\in [n]}\lFD_{it}  &\leq \sum_{i\in [n]}\max\{\lPE_{it}, \lSI_{it}\} \\
&\leq2\beta_{3t}\sum_{i\in [n]}\rho_{it}(\vx_{it}'') + \beta_{3t}\sum_{i\in [n]}\rho_{it}(\vx_{it})\\
    &\leq \frac{4}{q_0} \beta_{3t} \E[\sum_{i\in [n]}\rho_{it}(\vx_{it})] + \beta_{3t} \sum_{i\in [n]}\rho_{it}(\vx_{it}).
\end{align*}

Therefore, we have 
\begin{align*}
    \mathbb{E}_t[\sum_{i\in [n]}\lFD_{it}] \leq \mathbb{E}_t[\sum_{i\in [n]}\bPE_{it}] & \leq \frac{4}{q_0} \beta_{3t} \E[\sum_{i\in [n]}\rho_{it}(\vx_{it})] + \beta_{3t}\E_t[\sum_{i\in [n]}\rho_{it}(\vx_{it})] + n(\delta_{t} + \delta_{2t})\\
    &\leq \frac{5}{q_0}\beta_{3t} \mathbb{E}_t[\sum_{i\in [n]}\rho_{it}(\vx_{it})] + n(\delta_{t} + \delta_{2t}),
\end{align*}
which completes the proof.
\end{proof}

With the above lemmas at hand, we are now ready to provide a proof of Theorem~\ref{thm:TS-anytime} for $\LFD_T$.

\subsection{Proof of Theorem~\ref{thm:TS-anytime} for $\LFD$}\label{app:proof-TS-anytime-LFD}

\begin{proof}
Lemma~\ref{lem:lemma6-FD} shows a analog guarantee as we obtained in lemma~\ref{lem:lemma6} for the $\LFD$ loss function. Therefore, following the same steps in lemma~\ref{lem:lemma7}, we have that with probability at least $1-\delta'$ where $\delta'$ will be specified momentarily,
\begin{align}\label{eq:LPE-bound}
\begin{split}
\LFD_T &= \sum_{t=t_0}^T \sum_{i=1}^n \lFD_{it} \\
&\leq n\sum_{t=1}^T (\delta_{t} + \delta_{2t}) + \frac{n\beta_{3t}}{q_0}\frac{\sqrt{T}}{\sqrt{M}}\bigg( 5c_{\phi}m \sqrt{ \frac{1}{\log \left(1 + \frac{\|\phi(\vone)\|_2}{\sqrt{M}} \right)}  \log \left(\frac{\|\phi(\vone)\|_2^2T}{M}\right) }  \\
&\quad +7c_{\phi}^2 \sqrt{m} \sqrt{2\log(\frac{1}{\delta'})} \bigg)\\
&= n\sum_{t=1}^T (\delta_{t} + \delta_{2t}) + \tilde{O} \bigg( n\frac{m^2}{\sqrt{M}}\sqrt{T} \left( \log(T) + \log \left(\frac{1}{\delta'}\right) \right) \bigg),
\end{split}
\end{align}
Choose $\delta_{t} =\delta_{2t} = \frac{2\delta}{n\pi^2 t^2}$. Then,
\[
\sum_{t=1}^T (\delta_{t} + \delta_{2t}) \leq \frac{2}{3}\delta.
\]
Also choose $\delta' = \frac{\delta}{3}$, then by Eq~\eqref{eq:LPE-bound}, with probability at least $1-\delta$,
\[
 \LFD_T = O\left(\frac{nm^2}{\sqrt{M}}\sqrt{T}\left(\log\left(\frac{\delta}{3}\right) + \log(T))\right)\right).
\]
\end{proof}

\subsection{Proof of Theorem~\ref{thm:TS-finitehorizon} for $\LFD$}\label{app:proof-TS-finite-time-LFD}
\begin{proof}
The theorem results follow from eq~\eqref{eq:LPE-bound}. Choose $\delta_{1t} =\delta_{2t} = \delta' = \frac{1}{T}$ and denote the event where $ \LFD_T = O\left(\frac{nm^2}{\sqrt{M}}\sqrt{T}\left(\log\left(\frac{\delta}{3}\right) + \log(T))\right)\right)$ holds as $\mathcal E$. Then,
\[
\E[\LFD_T] = \E[\LPE_T|\cE] + \E[\LFD_T|\cE^c] P(\cE^c) \leq O\left(\frac{nm^2}{\sqrt{M}}\sqrt{T}\left(\log\left(\frac{\delta}{3}\right) + \log(T))\right)\right),
\]

where $M \geq m$. This completes the proof.

\end{proof}

\section{Additional Experimental Details and Results}\label{app:exp-additional}

In this section, we present the simulation results with the Amdahl utilities, as described in Section~\ref{sec:experiments}, and additional implementation details. 

Figure~\ref{fig:results-amd} presents the simulation results for agents with the Amdahl's utilities.

\begin{figure*}[!ht]
\centering
\begin{tabular}{ccc} 
\includegraphics[width=0.27\textwidth]{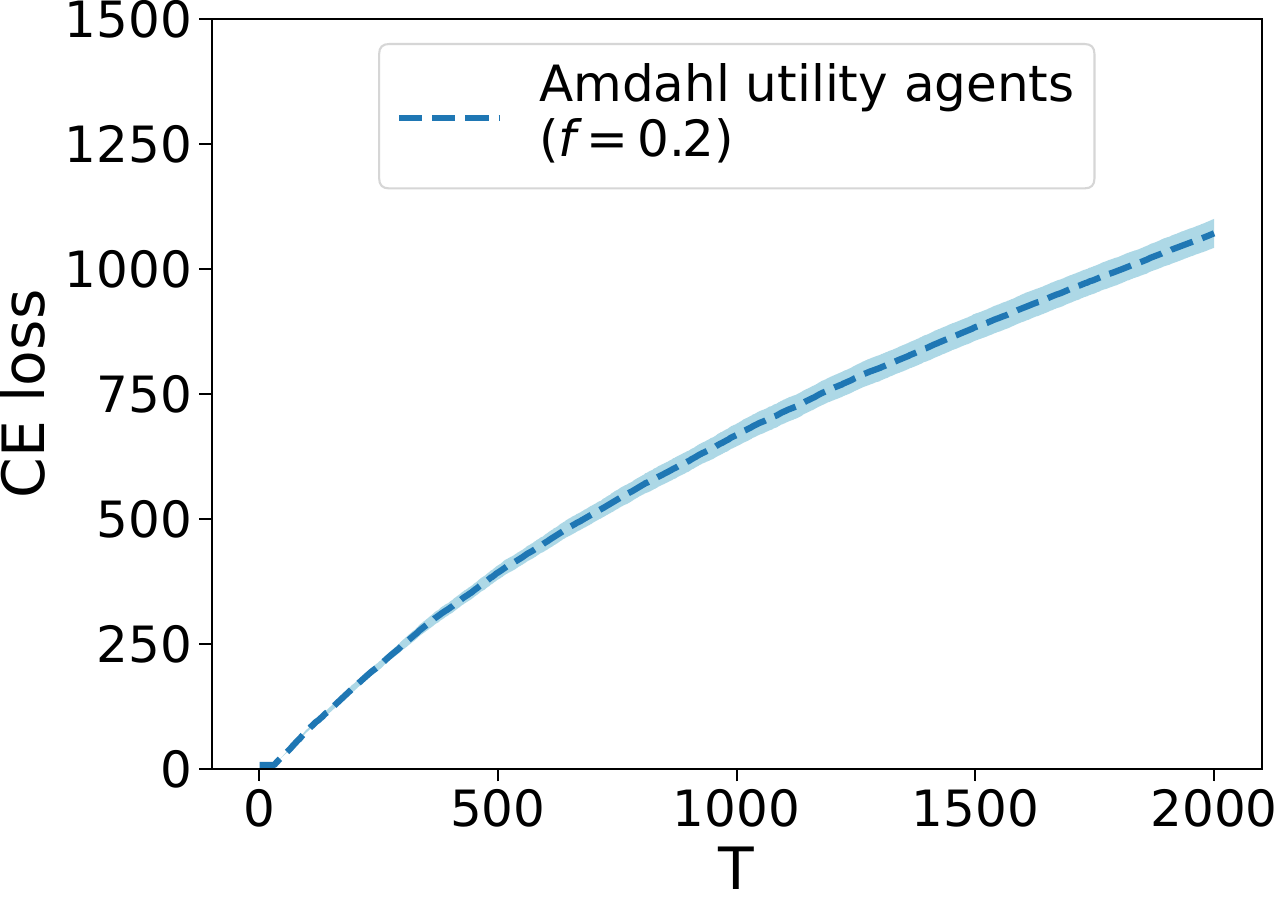} &
\includegraphics[width=0.27\textwidth]{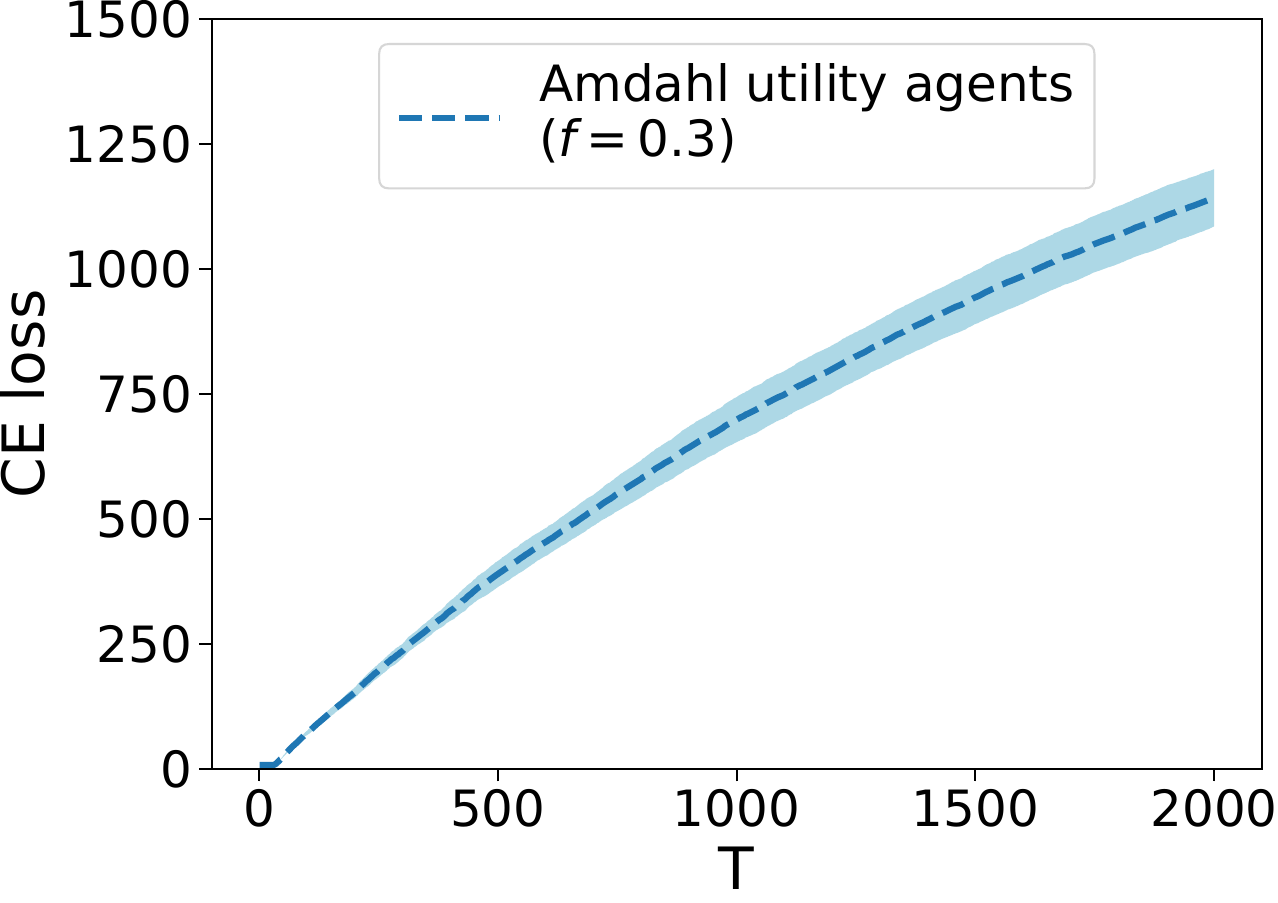} &
\includegraphics[width=0.27\textwidth]{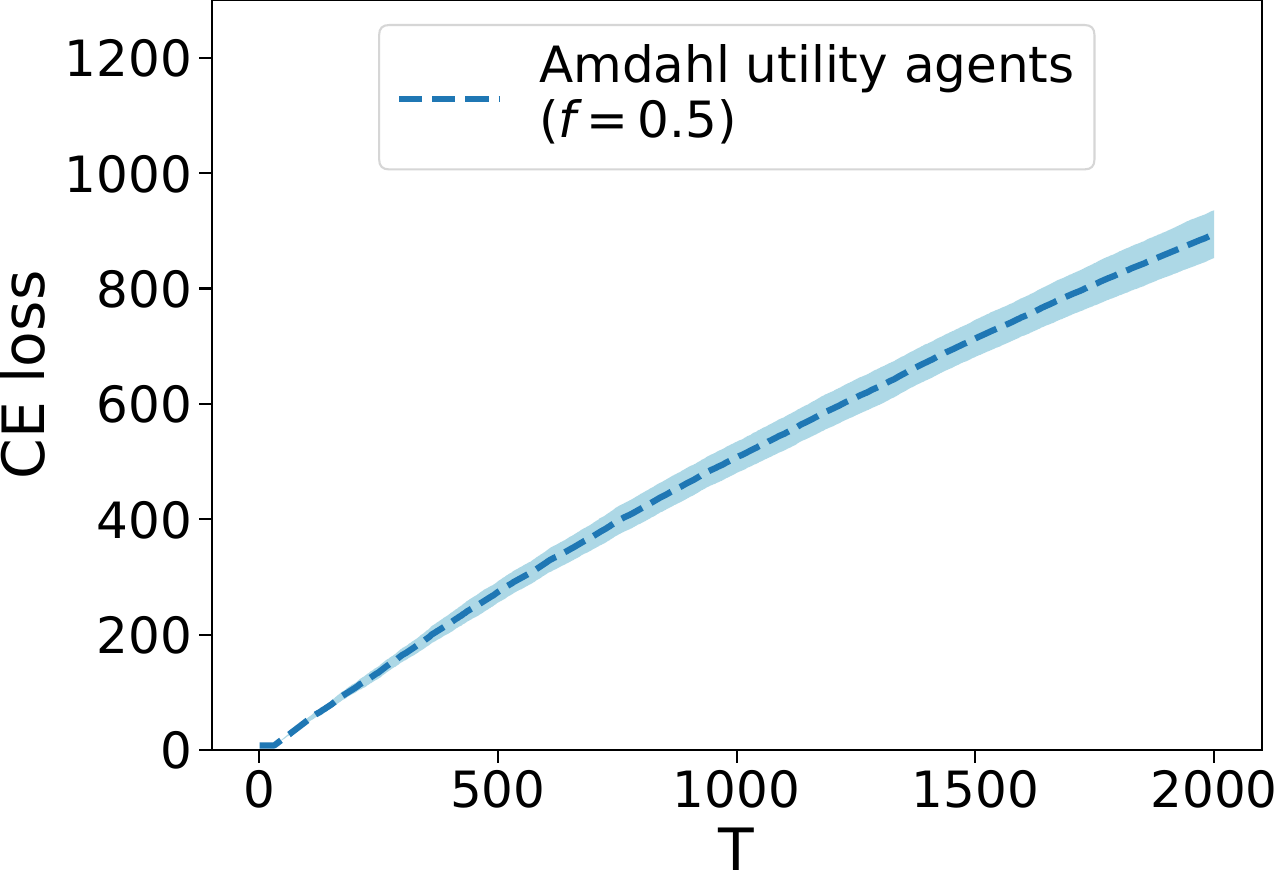}
\end{tabular}
\caption{
The CE loss $\LCE_T$ vs the number of rounds $T$. evaluated with $m=2$ resource types and $n=8$ agents and Amdahl's utilities.
The three figures correspond to $f_i=0.2$,
$f_i=0.3$, and $f_i=0.5$ (as in Section~\ref{sec:experiments}).
All figures show results which are averaged over 10 runs, and the shaded region shows the standard
error at each time $T$.
\label{fig:results-amd}
} 
\end{figure*}

Empirically, we compute $\LCE$ by maximizing each agent's utility subject to the budget constraint. We approximate this by randomly sampling feasible allocations $y$ from a simplex, accept those that cost no more price than the agent's endowment, lastly take the maximum. We sampled up to 50 accepted samples in each round. All experiments are run on a AWS EC2 p3.2xlarge machine.
\section{Further Discussions}\label{app:sec:additional-discussion}
\subsection{Further Background on Fair Division and Exchange Economies} \label{app:EE-background}

Since the seminal work of~\citet{varian1973equity}, fair division of multiple resource types has
received significant attention in the game theory, economics, and computer systems literature. We provide more background on the related works in the fair division literature and their applications in this section.

Among the theoretical works in fair division, one of the most common perspectives on this problem is as an exchange economy (or as a Fisher market,
which is a special case of an EE)~\citep{varian1973equity,mas1995microeconomic,
gutman2012, crockett2008learning,
tiwari2009competitive, budish2017course, babaioff2019fair, babaioff2021competitive}.
%

Fair allocation mechanisms have been deployed in many
practical resource allocation tasks when compute resources are shared by multiple users~\citep{zahedi2018amdahl}. There have also been applications of other market-based resource allocation schemes for data centers and power grids~\citep{lai2005tycoon,wolski2001analyzing, hindman2011mesos, kubernetes, yarn}. 
One line of work in this setting studied fair division when the resources in question are
perfect complements; some examples include dominant resource fairness
and its variants~\citep{ghodsi2011dominant,parkes2015beyond, gutman2012,ghodsi2013choosy, li2013egalitarian,dolev2012no}.
Although the assumption of perfect complement resource types leads to computationally simple mechanisms, in many practical applications, there is ample substitutability between resources, and hence the above mechanisms can be inappropriate.
For example, in compute clusters, CPUs and GPUs are often interchangeable for many jobs, albeit
with different performance characteristics. Indeed, in this work, we in particular focused on the applications of EE and fair division mechanisms for computing resource allocation.

\subsection{Computation of the FD loss}

We note that one main challenge for computing the FD loss is that we need to approximate the Pareto Front and then take a minimum over it. To approximate it, even in the simplest two agent two resource set up, this requires  grid search on a 4D space which can be computationally prohibitive. In our experiments, the dimensions are 15 (3 resources by 5 agents) and 16 (2 resources by 8 agents), for which grid search is not feasible. Given that efficient computations over the Pareto frontier have remain a technical challenge, we focused the evaluations on the CE loss in this work.

\subsection{On the loss functions}
\label{sec:apploss}

We provide an example to demonstrate that the FD loss~\eqref{eqn:fdloss}, while
more interpretable than the CE loss~\eqref{eqn:celoss}, may not capture all properties of
an equilibrium.

For this consider the following example with $n=3$ agents and $m=2$ resources where
the endowments of agent 1, agent 2, and agent 3 are $e_1=(0.45, 0.05)$, $e_2=(0.45, 0.05)$, and
$e_3=(0.1, 0.9)$ respectively.
Their utilities are:
\[
\util_1(x_1) = 0.1 x_{11} + x_{12},
\hspace{0.3in}
\util_2(x_2) = 0.2 x_{21} + x_{22},
\hspace{0.3in}
\util_3(x_3) = x_{31} + 0.1 x_{32}.
\]
The utilities of the three users if they were to simply use their endowment is,
$\util_1(e_1) = 0.1 \times 0.45 + 0.05 = 0.095$,
$\util_2(e_2) = 0.14$,
and
$\util_3(e_3) = 0.19$.
We find that while agents 1 and 2 benefit more from the second resource, they have more of the first
resource in their endowments and vice versa for agent 3.
By exchanging resources, we can obtain a more efficient allocation.

The unique equilibrium prices for the two goods are $\peq = (1/2, 1/2)$ and
the allocations are
$\xeq_1 = ( 0, 0.5)$ for agent 1,
$\xeq_2 = ( 0, 0.5)$ for agent 2,
and
$\xeq_3 = ( 1.0, 0.0)$ for agent 3.
The utilities of the agents under the equilibrium allocations are
$\util_1(x_1) = 0.5$,
$\util_2(x_2) = 0.5$,
and
$\util_3(x_3) = 1.0$.
Here, we find that by the definition of CE, $\lPE(\xeq, \peq) = 0$.
It can also be verified that $\lFD(\xeq, \peq) = 0$.

In contrast, consider the following allocation for the 3 users:
$x_1 = ( 0.35, 0.49)$ for agent 1,
$x_2 = ( 0.35, 0.49)$ for agent 2,
and
$x_3 = (0.3, 0.02)$ for agent 3.
Here, the utilities are
$\util_1(x_1) = 0.1 \times 0.35 + 0.49 = 0.525$,
$\util_2(x_2) = 0.56$,
and
$\util_3(x_3) = 0.3002$.
This allocation is both PE (as the utility of one user can only be increased by taking resources
from someone else),
and SI (as all three users are better off than having their endowments).
Therefore, $\lFD((x_1, x_2, x_3)) = 0$.
However, user 3 might complain that their contribution of resource 2 (which was useful for users 1
and 2) 
has not been properly accounted for in the allocation.
Specifically, there do not exist a set of prices $p$ for which $\lPE(x, p) = 0$.
This example illustrates the role of prices in this economy: it allows us to value the resources
relative to each other based on the demand.

\end{document}